\documentclass[lettersize,journal]{IEEEtran}

\usepackage[colorlinks,urlcolor=blue,linkcolor=blue,citecolor=blue]{hyperref}
\usepackage{amsmath,amsfonts}
\usepackage{amsthm}
\usepackage{array}
\usepackage{textcomp}
\usepackage{stfloats}
\usepackage{url}
\usepackage{verbatim}
\usepackage{graphicx}
\usepackage{cite}
\usepackage{booktabs}
\usepackage{multirow}
\usepackage{threeparttable}
\usepackage{dcolumn}
\usepackage{bm}
\usepackage{array}
\usepackage{enumerate}
\usepackage{siunitx}
\usepackage{subcaption}
\usepackage{caption}
\usepackage{float}
\usepackage{colortbl}
\usepackage{xcolor}
\usepackage{tabularx}
\usepackage{amssymb}
\usepackage[ruled,linesnumbered]{algorithm2e}
\newtheorem{theorem}{Theorem}

\begin{document}

\title{Learning Adaptive Cross-Embodiment Visuomotor Policy with Contrastive Prompt Orchestration}

\author{Yuhang Zhang, \IEEEmembership{Student Member, IEEE}, Chao Yan, \IEEEmembership{Member, IEEE}, Yu Jiaxi, Jiaping Xiao, \IEEEmembership{Member, IEEE}, and Mir Feroskhan, \IEEEmembership{Member, IEEE}
        
\thanks{Y. Zhang, J. Yu, J. Xiao, and M. Feroskhan are with the School of Mechanical and Aerospace Engineering, Nanyang Technological University, Singapore 639798, Singapore (e-mail: yuhang004@e.ntu.edu.sg; jiaping001@e.ntu.edu.sg; jiaxi002@e.ntu.edu.sg; mir.feroskhan@ntu.edu.sg). C. Yan is with the College of Automation Engineering, Nanjing University of Aeronautics and Astronautics, Nanjing, 211106, China (e-mail: yanchao@nuaa.edu.cn). \textit{(Corresponding authors: Mir Feroskhan, and Jiaping Xiao.)}
}
}

\markboth{Journal of \LaTeX\ Class Files,~Vol.~14, No.~8, August~2021}%
{Shell \MakeLowercase{\textit{et al.}}: A Sample Article Using IEEEtran.cls for IEEE Journals}


\maketitle

\begin{abstract}
Learning adaptive visuomotor policies for embodied agents remains a formidable challenge, particularly when facing cross-embodiment variations such as diverse sensor configurations and dynamic properties. Conventional learning approaches often struggle to separate task-relevant features from domain-specific variations (e.g., lighting, field-of-view, and rotation), leading to poor sample efficiency and catastrophic failure in unseen environments. To bridge this gap, we propose ContrAstive Prompt Orchestration (CAPO), a novel approach for learning visuomotor policies that integrates contrastive prompt learning and adaptive prompt orchestration. For prompt learning, we devise a hybrid contrastive learning strategy that integrates visual, temporal action, and text objectives to establish a pool of learnable prompts, where each prompt induces a visual representation encapsulating fine-grained domain factors. Based on these learned prompts, we introduce an adaptive prompt orchestration mechanism that dynamically aggregates these prompts conditioned on current observations. This enables the agent to adaptively construct optimal state representations by identifying dominant domain factors instantaneously. Consequently, the policy optimization is effectively shielded from irrelevant interference, preventing the common issue of overfitting to source domains. Extensive experiments demonstrate that CAPO significantly outperforms state-of-the-art baselines in sample efficiency and asymptotic performance. Crucially, it exhibits superior zero-shot adaptation across unseen target domains characterized by drastic environmental (e.g., illumination) and physical shifts (e.g., field-of-view and rotation), validating its effectiveness as a viable solution for cross-embodiment visuomotor policy adaptation.
\end{abstract}


\section{Introduction}

\begin{figure}[t!]
    \centering
    
    \begin{subfigure}{\linewidth}
        \centering
        \includegraphics[width=1.0\linewidth]{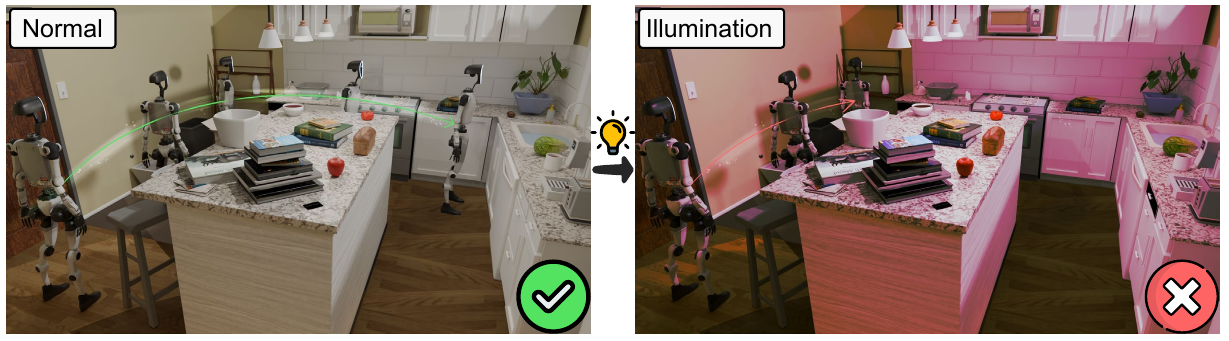}
        \caption{Illumination Sensitivity}
        \label{fig:a}
    \end{subfigure}
    
    \begin{subfigure}{\linewidth}
        \centering
        \includegraphics[width=1.0\linewidth]{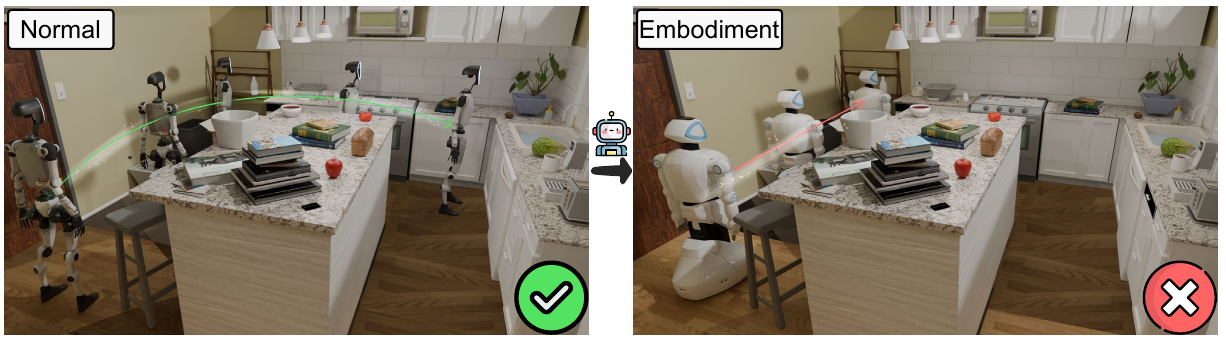}
        \caption{Embodiment Sensitivity}
        \label{fig:b}
    \end{subfigure}
    
    \begin{subfigure}{\linewidth}
        \centering
        \includegraphics[width=1.0\linewidth]{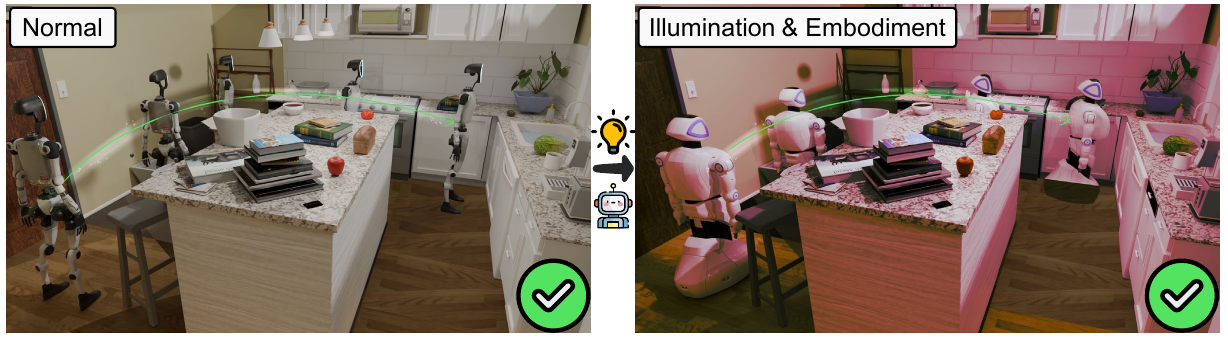}
        \caption{Joint Illumination and Embodiment Adaptation}
        \label{fig:c}
    \end{subfigure}
    
    \caption{Illustration of domain sensitivity in visuomotor policy learning and the adaptability of the proposed CAPO approach. (a) Conventional learning-based policies are highly sensitive to illumination. An agent trained in a standard environment fails to adapt when facing drastic illumination shifts, even if the underlying scene layout remains unchanged. (b) Similarly, standard approaches struggle with embodiment variations. A policy trained on a specific robot configuration fails when deployed on a different embodiment (e.g., varying FOV). (c) In contrast, CAPO synergizes representation learning and policy optimization via the adaptive orchestration of contrastive prompts. This design enables superior zero-shot adaptability, allowing the agent to successfully navigate in unseen target domains with joint illumination and embodiment variations.}

    \label{fig:head}
\end{figure}

Recent advances in embodied artificial intelligence (AI) have served as a solid bedrock for robot learning, enabling embodied agents to perform various real-time decision-making tasks in complex environments, such as navigation \cite{xiao2025vision}, manipulation \cite{kong2025pixel}, and locomotion \cite{srisuchinnawong2025interpretable}. As a promising paradigm, vision-based embodied agents have received increasing attention in the past few years, relying on onboard cameras to perceive the surrounding environment and generate action commands. However, as illustrated in Fig. \ref{fig:head}, visual observations are inherently entangled with both environmental conditions and the agent’s physical configuration, introducing significant uncertainty into perception and decision-making. In particular, agents’ sensory inputs are highly sensitive to environmental and physical factors, such as illumination, field of view (FOV), and stride length, which can substantially alter the agent’s visual observations. These factors, known as \textit{domain factors}, pose major challenges to learning adaptive visuomotor policies. This raises a fundamental question: How can we develop a policy that enables an agent to adapt to varying domain factors?

A prevalent paradigm for visuomotor policy learning relies on large-scale expert demonstrations and imitation learning (IL), where policies are trained to replicate expert behavior in tasks such as autonomous driving \cite{wu2023human} and robotic manipulation \cite{belkhale2023hydra}. However, collecting high-quality real-world demonstrations is costly, labor-intensive, and difficult to scale. Reinforcement learning (RL) offers a scalable alternative by enabling policy training through interaction in high-fidelity simulators, allowing efficient data collection and rapid iteration \cite{zare2024survey}. Many RL approaches learn policies in an end-to-end manner by directly mapping raw visual inputs to actions \cite{wu2023human, zhang2024npe, zhang2025learning1}. Despite their effectiveness, they often require millions of interactions to converge and commonly suffer from sample inefficiency. Moreover, learning directly from raw visual inputs makes them highly sensitive to domain-induced variations in pixel distributions, commonly referred to as \emph{domain gaps}. As a result, policies learned under specific visual conditions often fail to generalize to unseen environments, even when the underlying task structure remains unchanged. Some research resorts to domain randomization \cite{tobin2017domain, tobin2018domain, horvath2022object}, synthesizing diverse visual and physical variations in simulation to promote adaptation. While effective, such methods often require careful manual design of randomization distributions and tend to yield policies that trade optimality for adaptation, performing suboptimally in specific environments \cite{kumar2021rma}.

To address these challenges, prior work has increasingly turned to decoupled architectures \cite{laskin2020curl, stooke2021decoupling, zhang2025learning} for policy adaptation under domain shifts, where visual representation learning is separated from policy optimization. In particular, more recent efforts focus on visual representation learning, either by training generalized visual encoders \cite{zhao2023learning, xing2024contrastive, zhang2025oracle} or leveraging large-scale pretrained models \cite{majumdar2022zson, shah2023lm, zhang2025grounded} to extract domain-invariant features. Although these approaches improve robustness to distributional shifts, they typically produce static representations \cite{dorbala2023can}, limiting their ability to adaptively emphasize task-relevant features under unseen domain factors. This lack of adaptive flexibility often leads to inconsistent visual embeddings and subsequent policy degradation.

In this paper, we propose \textbf{C}ontr\textbf{A}stive \textbf{P}rompt \textbf{O}rchestration (\textbf{CAPO}), an adaptive cross-embodiment visuomotor policy learning apporach, to improve zero-shot adaptation under unseen domains. Our approach is motivated by combining the complementary strengths of end-to-end and decoupled paradigms, aiming to retain sample-efficient learning while overcoming the rigidity of static state representations to enable more effective adaptation under domain shifts. To achieve this, CAPO combines a contrastive prompt learning strategy with an adaptive orchestration mechanism, balancing representation robustness and policy adaptability. In the contrastive prompt learning, we employ CLIP \cite{radford2021learning} as a vision-language backbone and introduce learnable prompts to extract multiple domain-aware visual representations from the same observation, without fine-tuning the encoder. These prompt-induced representations are trained using a hybrid contrastive strategy, where visual, temporal action, and text contrastive objectives are applied across different prompts to guide their optimization. After this, to enable adaptation to dynamic domain factors, we design an adaptive orchestration mechanism based on a dual-branch attention architecture that dynamically aggregates prompted embeddings conditioned on the current observation. The aggregated representation is then fed into the policy network, where both the orchestration parameters and the policy network are jointly optimized via RL. This design allows the trained policy to adapt seamlessly from the \textit{source domain} (where the policy is trained) to the \textit{target domain} (where the policy is deployed). By synergizing the stability of pre-learned prompts with the flexibility of adaptive orchestration, our approach significantly enhances sample efficiency and adaptation performance.


The contributions of this paper are summarized below.
\begin{enumerate}
    \item We propose CAPO, a novel visuomotor policy learning approach that addresses the limitations of end-to-end and decoupled RL by retaining sample-efficient learning while overcoming static representations through integrated prompt learning and adaptive orchestration. It enables the dynamic construction of visual embeddings, allowing effective adaptation to unseen domain factors without fine-tuning.
    \item We design a hybrid contrastive learning strategy that integrates visual, temporal action, and text supervision for contrastive prompt learning. This hybrid formulation allows the encoder to extract richer and more adaptive features than standard single-objective approaches, forcing the encoder to encapsulate fine-grained domain factors.
    \item We introduce a dual-branch attention-based orchestration mechanism that dynamically re-weights prompt-induced embeddings conditioned on the current observation, enabling effective adaptation to unseen target domains and embodiment variations.
    \item We conduct extensive experiments to evaluate CAPO. The results show that it outperforms state-of-the-art baselines in terms of sample efficiency and asymptotic performance. Moreover, CAPO exhibits superior zero-shot policy adaptation when deployed in target domains, involving unseen embodiment or environmental factors.
\end{enumerate}

The remainder of this paper is organized as follows. Section \uppercase\expandafter{\romannumeral2} reviews related work on visuomotor policy learning, policy adaptation for embodied agents, and prompt learning. Section \uppercase\expandafter{\romannumeral3} presents the problem formulation and preliminaries. Section \uppercase\expandafter{\romannumeral4} describes the proposed CAPO approach in detail. Section \uppercase\expandafter{\romannumeral5} outlines the experimental setup. Section \uppercase\expandafter{\romannumeral6} reports the results in simulated environments and discusses the ablation studies. Finally, Section \uppercase\expandafter{\romannumeral7} concludes the paper and discusses directions for future research.

\section{Related Work}
This section categorizes related work into three primary areas: visuomotor policy learning, policy adaptation for embodied agents, and prompt learning.

\subsection{Visuomotor Policy Learning}
Visuomotor policy learning is a popular task for embodied agents, where the objective is to map high-dimensional visual inputs to low-dimensional, executable actions. Existing approaches can be broadly categorized into IL and RL paradigms. IL methods \cite{loquercio2018dronet, pan2020imitation, ramrakhya2022habitat, zhao2023learning1} typically rely on large-scale and high-quality datasets collected from real-world environments, where DNNs are trained to imitate expert demonstrations and approximate the pixel-to-action mapping. For instance, prior works have leveraged human trajectories in photorealistic simulators \cite{ramrakhya2022habitat}, accident datasets \cite{loquercio2018dronet}, and teleoperated robot demonstrations \cite{zhao2023learning1} to train agents that imitate human-like navigation strategies. However, collecting such datasets is time-consuming, costly, and inherently difficult to scale. Moreover, IL methods are often plagued by distribution shift \cite{ross2011reduction}, where minor deviations from the expert distribution lead to compounding errors, resulting in poor adaptation to unseen environments.

To mitigate these limitations, RL-based approaches have gained increasing attraction \cite{kaufmann2023champion, zhang2024npe, xiao2023collaborative, yan2023collision, zhang2025learning1}, which can be broadly divided into end-to-end and decoupled paradigms. By interacting with the environment and learning from trial-and-error feedback, RL enables agents to learn policies that are more adaptive to diverse and dynamic settings. End-to-end learning approaches directly map visual inputs to action commands. For instance, Kaufmann et al. \cite{kaufmann2023champion} developed a deep RL framework for autonomous drone racing that learns agile visuomotor policies directly from raw pixel input, achieving champion-level performance. Zhang et al. \cite{zhang2025learning1} further improved sample efficiency and control accuracy by integrating differentiable physics into the RL training loop. Despite their success, such approaches often suffer from prohibitive sample complexity, as learning rich visual representations solely from sparse scalar rewards is notoriously inefficient \cite{stooke2021decoupling}. Furthermore, these coupled architectures struggle to adapt when there are shifts in embodiment or sensor configurations.

To alleviate this inefficiency, a distinct line of research advocates for decoupled visuomotor learning \cite{laskin2020curl, stooke2021decoupling, zhang2025learning}. In this paradigm, visual encoders are typically pre-trained via self-supervised objectives, and subsequently frozen to provide stable state representations for downstream policy optimization. While these decoupled architectures significantly improve sample efficiency by isolating representation learning from control dynamics, they often result in rigid and static representations. Since the encoder is agnostic to the downstream task dynamics and remains fixed during deployment, it lacks the flexibility to adapt when there are drastic shifts in environment or embodiment configurations.

In contrast to these efforts, our approach bridges visual representation learning and policy optimization via an adaptive orchestration mechanism, effectively synergizing the stability of decoupled RL architectures with the flexibility of end-to-end optimization. We aim to extract adaptive visual features that are conditioned on domain factors, which are then dynamically composited for policy learning. This design facilitates better sample efficiency and enables improved adaptation across previously unseen visual domains or embodiment variations.

\subsection{Policy Adaptation for Embodied Agents}
For embodied agents, policy adaptation techniques aim to adapt the learned policy to environments with varying domain factors. Early approaches \cite{tobin2017domain, tobin2018domain, horvath2022object} extensively employ domain randomization for end-to-end learning to synthesize visual textures and physical parameters within simulation. These methods expose agents to various randomized scenarios in an attempt to promote generalizable behaviors, but they often require substantial manual tuning and tend to adopt overly conservative behaviors in the specific scenario \cite{kumar2021rma}.

Consequently, a growing body of research has shifted focus towards the decoupled paradigm of visual representation learning, which strives to extract domain-invariant features offline to be resilient to distributional shifts. Current approaches generally fall into two categories: learning generalized visual encoders \cite{zhao2023learning, xing2024contrastive, zhang2025oracle} or leveraging large-scale pretrained models \cite{majumdar2022zson, shah2023lm, zhang2025grounded}. The first category involves training custom encoders, often utilizing contrastive learning paradigms. The core mechanism involves optimizing the latent space by aligning embeddings of semantically related samples (positive pairs) while distancing unrelated ones (negative pairs), thereby enhancing the discriminative capacity of the representation. Alternatively, recent works have sought to harness the rich semantic priors of large-scale pretrained models (e.g., CLIP \cite{radford2021learning}) directly. Compared to training custom encoders from scratch, these large models offer superior zero-shot robustness and data efficiency. For instance, Zhang et al. \cite{zhang2025grounded} utilized CLIP as a goal retrieval module to align agent observations with human language instructions. This alignment enables the agent to identify task-relevant targets in complex, unstructured environments, effectively facilitating open-vocabulary navigation tasks without extensive task-specific fine-tuning.


Despite their effectiveness, both paradigms exhibit limitations in dynamic adaptation. Custom encoders and frozen pretrained backbones typically yield static representations \cite{dorbala2023can}, and the feature extraction process remains fixed regardless of the domain shifts encountered during policy learning. In contrast, our approach integrates the strengths of both approaches. While we build upon the robust foundation of large-scale pretrained models, we introduce learnable prompts optimized via contrastive learning. This design allows our system to adaptively generate representations, thereby addressing the adaptability deficiency of static encoders when facing significant domain gaps.

\subsection{Prompt Learning}
Prompt learning provides a parameter-efficient alternative to full model fine-tuning by introducing lightweight, learnable prompts that condition large pretrained models toward downstream tasks. It is initially explored in language models through optimizing a small set of textual prompt vectors \cite{zhou2022learning}. It is also extended to vision transformers, where inserting learnable prompts into transformer layers enables broad task adaptation without modifying backbone weights \cite{yang2023fine}, which is recognized as visual prompting. Recently, this concept has expanded to visual–textual prompting in multi-modal embedding spaces, demonstrating that prompts can encode complementary task-specific cues \cite{khattak2023maple, yang2024fine}. For instance, Khattak et al. \cite{khattak2023maple} proposed multi-modal prompt learning, which introduces learnable prompts into both vision and language branches to ensure mutual synergy and improve the alignment between multi-modal representations. Similarly, Yang et al. \cite{yang2024fine} investigated fine-grained visual prompting for instance-level perception, devising a blur reverse mask strategy that suppresses background noise while retaining spatial coherence to enhance zero-shot comprehension. Despite these successes in static vision tasks, the application of prompt learning in embodied agents remains underexplored.

 In this paper, we bridge this gap by adopting visual prompting techniques to explicitly address the perceptual challenges faced by embodied agents. Unlike prior work, we exploit diverse hybrid contrastive prompts to fundamentally restructure the representation space, enabling zero-shot adaptation in complex cross-embodiment visuomotor policy learning scenarios.


\section{Preliminaries and Problem Formulation}

This section outlines the preliminaries of CAPO, providing the problem formulation and the contrastive learning objective designed to facilitate visuomotor policy training.

\subsection{Problem Formulation}

We formulate the visuomotor task for embodied agents as a Partially Observable Markov Decision Process (POMDP), defined by a tuple $\mathcal{M} = (\mathcal{S}, \mathcal{O}, \mathcal{A}, \mathcal{T}, \mathcal{R}, \gamma)$. Here, $\mathcal{S}$ represents the underlying state space of the environment, which is not directly accessible to the agent. $\mathcal{O}$ denotes the high-dimensional observation space, where the agent receives an egocentric visual observation $o_t \in \mathcal{O}$ at each time step $t$, determined by the conditional observation probability $\mathcal{P}(o_t \mid s_t)$ given the hidden state $s_t \in \mathcal{S}$. The action space $\mathcal{A}$ consists of the navigation commands executable by the agent. The transition dynamics $\mathcal{T}(s_{t+1} \mid s_t, a_t)$ govern the evolution of the environment state given the current state and the executed action $a_t \in \mathcal{A}$. The objective is defined by a reward function $\mathcal{R}: \mathcal{S} \times \mathcal{A} \rightarrow \mathbb{R}$, which provides a scalar feedback $r_t$ reflecting the task progress, and $\gamma \in [0, 1)$ is the discount factor.

Unlike mainstream end-to-end approaches that train a visual encoder $f_\theta$ from scratch \cite{zhang2024npe, xiao2023collaborative, yan2023collision}, we decouple visual representation learning from policy optimization by leveraging CLIP as the image encoder, denoted as $\Phi$. To adapt this fixed backbone to the dynamic domain factors of embodied navigation, we introduce a set of $K$ learnable visual prompts, denoted as $\mathbf{P} = \{ \boldsymbol{p}^1, \boldsymbol{p}^2, \dots, \boldsymbol{p}^K \}$. Each prompt $\boldsymbol{p}^k$ is defined as a sequence of continuous learnable vectors:
\begin{equation}
    \boldsymbol{p}^k = [\boldsymbol{v}^k_1, \boldsymbol{v}^k_2, \dots, \boldsymbol{v}^k_L], \quad \boldsymbol{v}^k_i \in \mathbb{R}^{D},
\end{equation}
where $L$ represents the prompt length and $D$ is the embedding dimension. These prompts are prepended to the input observation $o_t$ and processed by the encoder $\Phi$ to extract context-aware visual features.

To construct the final state representation from these multiple prompted embeddings, we employ an attention-based prompt orchestration mechanism $\mathcal{G}_{\text{attn}}$. Let $\mathbf{z}_t^{v} = \Phi(o_t, \varnothing)$ be the vanilla embedding using the frozen CLIP encoder $\Phi$ without prompts and $\boldsymbol{z}^k_t = \Phi(o_t, \boldsymbol{p}^k)$ be the feature embedding derived from the $k$-th prompt. The final fused feature representation $\boldsymbol{z}_t^f$ is computed as a weighted combination of these embeddings:
\begin{equation}
    \boldsymbol{z}_t^f = \mathbf{z}_t^{v} + \mathcal{G}_{\text{attn}}(\{ \boldsymbol{z}^1_t, \dots, \boldsymbol{z}^K_t \}),
\end{equation}
where the attention weights are dynamically determined based on the current visual observation (the details will be introduced in the subsequent section). This fused representation $\boldsymbol{z}_t^f$ then serves as the input to the policy network $\pi_\psi$, parameterized by $\psi$. The goal of the agent is to learn an optimal policy $\pi_\psi^*$ that maximizes the expected discounted cumulative reward: $J(\pi_\psi) = \mathbb{E}_{\pi_\psi} [\sum_{t=0}^{T} \gamma^t r_t]$.

\subsection{Contrastive Learning}
Contrastive learning has emerged as a dominant paradigm in self-supervised representation learning, which is pivotal for effective policy adaptation, aimed at extracting discriminative features by maximizing the mutual information between semantically similar instances. Instead of relying on manual annotations, it treats instance discrimination as a dictionary look-up task. Formally, given a batch of sampled observations, we employ a query encoder $f_q$ and a key encoder $f_k$ to map inputs into a shared latent space, producing feature vectors $q$ and $k$. The training objective optimizes the alignment between a positive pair $(q_i, k_i)$ while contrasting it against negative samples. Consequently, the optimization is performed by minimizing the contrastive loss \cite{laskin2020curl}:
\begin{equation}
\label{infonce}
\mathcal{L}_{\text{NCE}} = - \mathbb{E} \left[ \log \frac{\mathrm{sim}(q_i, k_i)}{\sum_{j=1}^{N} \mathrm{sim}(q_i, k_j)} \right],
\end{equation}
where the similarity function is defined as $\mathrm{sim}(u, v) = \exp(u^T v)$, with $u$ and $v$ denoting $\ell_2$-normalized feature vectors. Here, we extend this formulation to the navigation task by employing a hybrid contrastive strategy (see Section \ref{sec:representation}).

\begin{figure*}[!t]
   \centering
   \includegraphics[width=0.95\textwidth]{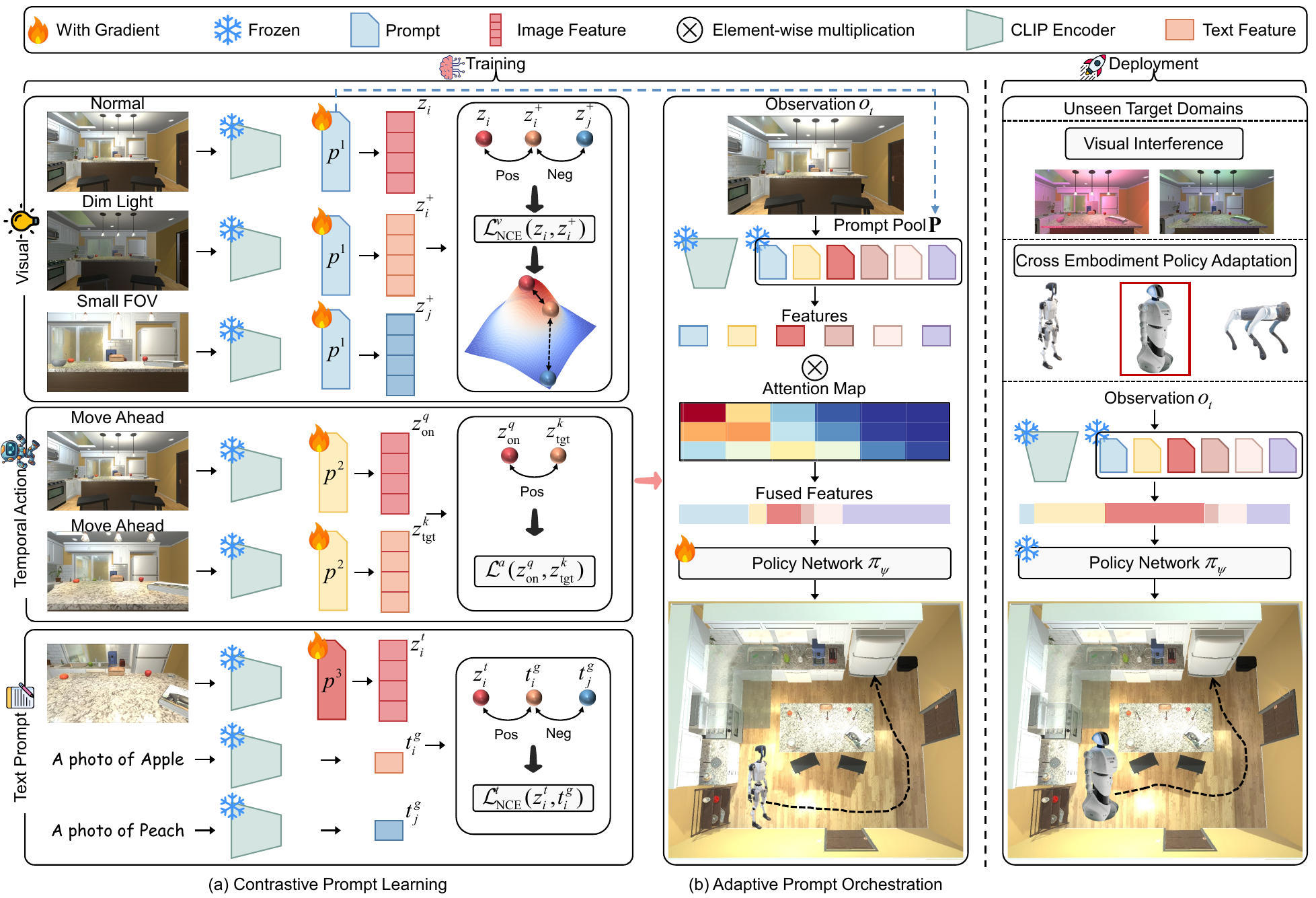} 
\caption{The framework of contrastive prompt orchestration (CAPO). (a) Contrastive Prompt Learning. We employ a hybrid contrastive learning strategy (incorporating visual, temporal action, and text objectives) to learn a pool of domain-specific visual prompts $\mathbf{P}$ for the frozen CLIP encoder $\Phi$. (b) Adaptive Prompt orchestration. We introduce an adaptive prompt orchestration mechanism that dynamically aggregates these frozen prompts based on the observation $o_t$, generating context-aware features for the policy network $\pi_{\psi}$. The trained policy achieves zero-shot cross-embodiment adaptation in unseen domains by automatically re-weighting prompts to handle visual interference and embodiment variations without fine-tuning.
}
   \label{fig:framework}
\end{figure*}

\section{Methodology}
This section details the proposed Contrastive Prompt Orchestration (CAPO) approach, including contrastive prompt learning, hybrid contrastive learning, adaptive prompt orchestration, and visuomotor policy learning.

\subsection{Overview}
We introduce CAPO, a novel visuomotor policy learning approach for embodied agents, aiming to enable zero-shot adaptation under unseen domain factors and cross-embodiment configurations. As illustrated in Fig.~\ref{fig:framework}, the overall training pipeline is organized into two tightly coordinated phases: (a) contrastive prompt learning and (b) adaptive prompt orchestration. In the prompt learning phase illustrated in Section~\ref{sec:representation}, we adopt a pretrained CLIP model \cite{radford2021learning} as the visual feature extractor and, instead of fine-tuning the entire backbone, freeze its image encoder while introducing a set of learnable visual prompts. Prompt learning is realized through a hybrid contrastive learning strategy that integrates visual, temporal action, and text contrastive learning objectives. This hybrid contrastive formulation forms a diverse prompt pool. In the prompt orchestration phase presented in Section~\ref{sec:policy learning}, we freeze the learned prompt pool and introduce an adaptive prompt orchestration mechanism coupled with policy learning. Given the current observation, this mechanism dynamically constructs a state representation by adaptively weighting prompt-conditioned features. The resulting representation is then fed into a policy network, which is trained via an RL algorithm to perform the visuomotor task. After training, the learned policy can be directly deployed for zero-shot cross-embodiment visuomotor adaptation by dynamically re-weighting prompt-induced embeddings conditioned on the observation. As a result, the policy can adjust the fused visual representation to domain-specific variations without requiring additional fine-tuning. The complete algorithm and the supporting theoretical analysis are summarized in Section~\ref{text:algorithm}.

\subsection{Contrastive Prompt Learning}
\label{sec:representation}

Here, we detail the prompt learning phase of our approach, which learns domain-specific visual prompts through hybrid contrastive learning. This phase produces a collection of visual prompts $\mathbf{P} = \{ \boldsymbol{p}^1, \boldsymbol{p}^2, \dots, \boldsymbol{p}^K \}$, each specialized for invariance to a particular domain factor. These prompts are subsequently frozen and utilized in the adaptive prompt orchestration phase. We describe this process from three aspects: data collection, visual prompt learning, and hybrid contrastive learning. Here, contrastive prompt learning focuses on optimizing prompt parameters, while hybrid contrastive learning serves as the training objective by computing contrastive losses over visual embeddings induced by different prompts.

\textbf{Data Collection.} 
We collect our training dataset $\mathcal{D}$ by executing expert navigation policies within the AI2-THOR simulation environment \cite{kolve2017ai2} under a range of domain and embodiment configurations. Each domain factor $\mathcal{F}_i$ is characterized by a tuple of environment and embodiment parameters:
\begin{equation}
\mathcal{F}_i = (\phi_i, \psi_i, \theta_i, \delta_i, \ell_i^b, \ell_i^c, \ell_i^s, \ell_i^h),
\end{equation}
where $\phi_i$ denotes the horizontal FOV parameter, $\psi_i$ refers to the rotation angle, $\theta_i$ is the look-up/down angle, $\delta_i$ specifies the translational step size, and $\ell_i^b, \ell_i^c, \ell_i^s, \ell_i^h$ represent lighting parameters controlling brightness, contrast, saturation, and hue. These factors are systematically varied to generate trajectories under diverse visual and embodiment conditions.

Specifically, for each configuration $\mathcal{F}_i$, a pretrained expert policy $\pi_e$ is deployed to navigate towards randomly sampled target objects. Each navigation episode yields a temporally ordered trajectory $\tau_i^j = \{(o_t^{i,j}, a_t^{i,j}, r_t^{i,j}, g^j)\}_{t=1}^{T_{i,j}}$, where $o_t^{i,j} \in \mathbb{R}^{C \times H \times W}$ denotes the RGB observation at timestep $t$, $a_t^{i,j} \in \mathcal{A}$ represents the executed action, $r_t^{i,j} \in \mathbb{R}$ specifies the reward value, and $g^j \in \mathcal{C}$ is the target category associated with the $j$-th trajectory.

To facilitate cross-domain contrastive prompt learning, we establish correspondence between trajectories collected under different factors that share the same navigation objective. Given a reference factor $\mathcal{F}_{\mathrm{b}}$ and an alternative factor $\mathcal{F}_i$, we identify aligned trajectory pairs $(\tau_{\mathrm{b}}^j, \tau_i^j)$ corresponding to the same target object $g^j$. Alignment is determined based on semantic similarity of the initial observations. For each trajectory, we compute an object presence vector as follows:
\begin{equation}
\mathbf{m} = (m_1, m_2, \dots, m_{|\mathcal{C}|}) \in \{0,1\}^{|\mathcal{C}|},
\end{equation}
where each element indicates the presence of a corresponding object category in the agent's initial semantic observation. Two trajectories are considered aligned if the F1-score between their object presence vectors exceeds a predefined threshold $\kappa$, ensuring that paired trajectories originate from semantically similar states under different domain factors. This alignment procedure yields a refined dataset organized as
\begin{equation}
\mathcal{D} = \bigcup_{i=1}^{K} \mathcal{D}_i, \quad \mathcal{D}_i = \{\tau_i^1, \tau_i^2, \ldots, \tau_i^{N}\},
\end{equation}
where each subset $\mathcal{D}_i$ contains trajectories collected under the same domain factor, and $N$ is the number of trajectories.

\textbf{Visual Prompt Learning.}
We leverage the rich semantic priors embedded in a pretrained CLIP model \cite{radford2021learning} to extract visual features from RGB observations. To adapt this pretrained visual encoder to domain-specific visuomotor tasks in a parameter-efficient manner, instead of updating all parameters of the pretrained model, we introduces a set of learnable prompt embeddings \cite{yang2023fine} that modulate the visual representations while keeping the CLIP image encoder frozen.

Formally, as mentioned before, we define a set of visual prompts $\mathbf{P} = \{ \boldsymbol{p}^1, \boldsymbol{p}^2, \dots, \boldsymbol{p}^K \}$, where each prompt $\boldsymbol{p}^k$ is a sequence of learnable vectors $\boldsymbol{p}^k = [\boldsymbol{v}^k_1, \boldsymbol{v}^k_2, \dots, \boldsymbol{v}^k_L]$ with $\boldsymbol{v}^k_i \in \mathbb{R}^{D}$. Given an RGB observation $o_t \in \mathbb{R}^{C \times H \times W}$ corresponding to the $k$-th domain factor, the CLIP encoder $\Phi$ first applies a patch embedding operation to map the image into a sequence of $N_p$ visual tokens:
\begin{equation}
\mathbf{x}_{\text{patch}} = \mathrm{PatchEmbed}(o_t) \in \mathbb{R}^{N_p \times d},
\label{vpl1}
\end{equation}
where $d$ denotes the hidden feature dimension. A class token $\mathbf{x}_{\text{cls}} \in \mathbb{R}^{1 \times d}$ is prepended to the patch sequence, and positional embeddings $\mathbf{E}_{\text{pos}} \in \mathbb{R}^{(N_p+1) \times d}$ are added to incorporate sequential structure:
\begin{equation}
\mathbf{X}_0 = [\mathbf{x}_{\text{cls}}; \mathbf{x}_{\text{patch}}] + \mathbf{E}_{\text{pos}} \in \mathbb{R}^{(N_p+1) \times d}.
\label{vpl2}
\end{equation}

Then, we insert the prompt embeddings $\boldsymbol{p}^k$ corresponding to the current domain factor into the token sequence after the class token. Specifically, the prompt tokens are first projected from dimension $D$ to the hidden dimension $d$ via a learnable linear transformation $\mathbf{W}_p \in \mathbb{R}^{D \times d}$:
\begin{equation}
\tilde{\boldsymbol{p}}^k = \boldsymbol{p}^k \mathbf{W}_p \in \mathbb{R}^{L \times d},
\label{vpl3}
\end{equation}
and then concatenated with the visual tokens, yielding the augmented input sequence:
\begin{equation}
\mathbf{X}_0^{\mathcal{P}} = [\mathbf{x}_{\text{cls}}; \tilde{\boldsymbol{p}}^k; \mathbf{x}_{\text{patch}}] + \mathbf{E}_{\text{pos}}^{\mathcal{P}} \in \mathbb{R}^{(1+L+N_p) \times d},
\label{vpl4}
\end{equation}
where $\mathbf{E}_{\text{pos}}^{\mathcal{P}}$ represents the positional embeddings extended to accommodate the inserted prompt length $L$.

The augmented token sequence is processed by a Transformer encoder $\xi$ consisting of $J$ layers, producing a sequence of representations: 
\begin{equation}
\mathbf{X}_j^{\mathcal{P}} = \xi_j(\mathbf{X}_{j-1}^{\mathcal{P}}), \quad j = 1, \ldots, J.
\label{vpl5}
\end{equation}

Finally, the output representation is extracted from the class token position, projected via the CLIP projection head $\mathbf{W}_{\text{proj}} \in \mathbb{R}^{d \times d_{\text{out}}}$, and $\ell_2$-normalized to obtain the domain-adapted visual embedding:
\begin{equation}
\mathbf{z}^{k}_t = \frac{\mathbf{X}_J^{\mathcal{P}}[0] \mathbf{W}_{\text{proj}}}{\| \mathbf{X}_J^{\mathcal{P}}[0] \mathbf{W}_{\text{proj}} \|_2} \in \mathbb{R}^{d_{\text{out}}}.
\label{vpl6}
\end{equation}

The embedding $\mathbf{z}^{k}_t$ serves as the input to the subsequent hybrid contrastive learning. Specifically, it is used to compute the contrastive losses, whose gradients are back-propagated to update the corresponding prompt embeddings $\boldsymbol{p}^k$. Through this optimization, the prompts progressively capture task-relevant invariances across domain factors, forming a robust set of basis representations.


\textbf{Hybrid Contrastive Learning.}
To build a diverse prompt pool $\mathbf{P}$ capable of handling diverse domain changes, we propose a hybrid contrastive learning strategy comprising three distinct objectives, namely visual contrastive learning, temporal action contrastive learning, and text contrastive learning.

\textbf{(i) Visual Contrastive Learning.}
It aims to enable the embodied agent to be robust against visual appearance changes, such as variations in brightness, contrast, saturation, and hue, which are distinct from physical embodiment changes. Specifically, we optimize the subset of prompts associated with lighting domain factors, denoted as $\{\boldsymbol{p}^{\ell_i^b}, \boldsymbol{p}^{\ell_i^c}, \boldsymbol{p}^{\ell_i^s}, \boldsymbol{p}^{\ell_i^h}\}$, to extract invariant representations despite visual perturbations. Here, $\ell_i^b, \ell_i^c, \ell_i^s, \ell_i^h$ denote continuous vectors representing the lighting configuration of the $i$-th domain factor.

For a prompt $\boldsymbol{p}^{\ell_i^b}$ corresponding to a specific lighting factor (e.g., brightness), we illustrate the construction of contrastive pairs; the same procedure applies to other lighting factors $\ell_i^c, \ell_i^s, \ell_i^h$. Given an anchor observation $o_i$ sampled from the dataset $\mathcal{D}$, we generate an augmented view $o_i^+ = \mathcal{V}_{\ell_i^b}(o_i)$ by perturbing the image according to the parameter space of $\ell_i^b$. Subsequently, we extract features for the anchor and positive samples using the frozen CLIP encoder $\Phi$ conditioned on the prompt $\boldsymbol{p}^{\ell_i^b}$ following Eqs.~\eqref{vpl1}--\eqref{vpl6}. The extracted features can be expressed as follows:
\begin{equation}
\boldsymbol{z}_i = \Phi(o_i, \boldsymbol{p}^{\ell_i^b}), \quad \boldsymbol{z}_i^+ = \Phi(o_i^+, \boldsymbol{p}^{\ell_i^b}),
\end{equation}
where $\boldsymbol{z}_i$ and $\boldsymbol{z}_i^+$ denote the visual embeddings of the original observation $o_i$ and its augmented view $o_i^+$, respectively, encoded with the same prompt $\boldsymbol{p}^{\ell_i^b}$.

To optimize the prompt, we employ a symmetric InfoNCE (Information Noise-Contrastive Estimation) objective following Eq.~\eqref{infonce}, treating the augmented view as the positive key and other samples in the batch $\mathcal{B}$ as negatives. The symmetric InfoNCE loss is defined by averaging both directions:
\begin{equation}
\mathcal{L}_{\text{NCE}}^{v}
=
\frac{1}{2}
\mathbb{E}_{i \in \mathcal{B}}
\left[
\mathcal{L}_{\text{NCE}}^{\text{uni}}(\boldsymbol{z}_i, \boldsymbol{z}_i^+)
+
\mathcal{L}_{\text{NCE}}^{\text{uni}}(\boldsymbol{z}_i^+, \boldsymbol{z}_i)
\right],
\end{equation}
where the single-direction InfoNCE loss is given by
\begin{equation}
\mathcal{L}_{\text{NCE}}^{\text{uni}}(\boldsymbol{z}_i, \boldsymbol{z}_i^+)
=
- \log
\frac{\mathrm{sim}(\boldsymbol{z}_i, \boldsymbol{z}_i^+)}
{\sum_{j=1}^{|\mathcal{B}|} \mathrm{sim}(\boldsymbol{z}_i, \boldsymbol{z}_j^+)} .
\end{equation}
Furthermore, to enforce tighter feature alignment, we incorporate a Mean Squared Error (MSE) regularization term:
\begin{equation}
\mathcal{L}_{\text{MSE}}^v = \frac{1}{|\mathcal{B}|} \sum_{i=1}^{|\mathcal{B}|} \|\boldsymbol{z}_i - \boldsymbol{z}_i^+\|_2^2.
\end{equation}
The final visual contrastive objective for prompt $\boldsymbol{p}^{\ell_i^b}$ is a weighted sum:
\begin{equation}
\mathcal{L}_{\text{visual}} = \mathcal{L}_{\text{NCE}}^v + \lambda_v \mathcal{L}_{\text{MSE}}^v,
\label{loss:visual}
\end{equation}
where $\lambda_v$ is a scaling hyperparameter that balances the alignment regularization strength.


\textbf{(ii) Temporal Action Contrastive Learning.}
This objective is designed to learn behavioral invariances across variations in agent configuration parameters, specifically the FOV, rotation angle, look-up angle, and step size. Accordingly, we optimize the subset of prompts associated with action execution factors, denoted as $\{\boldsymbol{p}^{u_i} \mid u_i \in \{\phi_i, \psi_i, \theta_i, \delta_i\}\}$, where these variables correspond to the aforementioned embodiment parameters. Unlike the visual contrastive learning approach which relies on negative samples, we adopt the BYOL self-distillation paradigm \cite{grill2020bootstrap} to avoid the difficulty of defining semantically meaningful negatives in the temporal action space.

For a specific action-conditioned prompt $\boldsymbol{p}^{u_i}$, we sample positive pairs $(o_q, o_k)$ from $\mathcal{D}$ that correspond to the same action types but are observed under different temporal contexts and embodiment configurations. To introduce view diversity, we apply stochastic image augmentations to generate augmented views $\tilde{o}_q$ and $\tilde{o}_k$. The learning framework consists of two parallel networks: an online network parametrized by $\omega$ and a target network parametrized by $\nu$. The online network comprises three components: the prompted encoder $\Phi$, a projector $\rho_{\omega}$, and a predictor $f_{\omega}$. The target network mirrors the encoder and projector architecture but excludes the predictor. The target parameters $\nu$ are updated as a momentum update rule \cite{laskin2020curl} of the online parameters $\omega$:
\begin{equation}
\nu \leftarrow \beta \nu + (1 - \beta) \omega,
\end{equation}
where $\beta \in [0, 1]$ is a decay rate.

We extract projected representations for the query and key views. The online network predicts the target representation of the cross-view sample. Specifically, let $\mathbf{z}_{\text{on}}^q = f_{\omega}(\rho_{\omega}(\Phi(\tilde{o}_q, \boldsymbol{p}^{u_i})))$ be the online prediction when the query view is input, and $\mathbf{z}_{\text{tgt}}^k = \texttt{SG}(h_{\nu}(\Phi(\tilde{o}_k, \boldsymbol{p}^{u_i})))$ be the target projection for the key view, where $\texttt{SG}(\cdot)$ denotes the stop-gradient operator. The asymmetric regression loss is defined as the MSE between the normalized prediction and target:
\begin{equation}
\mathcal{L}^a(\tilde{o}_q, \tilde{o}_k) = \left\| \frac{\mathbf{z}_{\text{on}}^q}{\|\mathbf{z}_{\text{on}}^q\|_2} - \frac{\mathbf{z}_{\text{tgt}}^k}{\|\mathbf{z}_{\text{tgt}}^k\|_2} \right\|_2^2.
\end{equation}
To enforce symmetry, we explicitly compute the loss in both directions: one where the online network processes $\tilde{o}_q$ to predict the target representation of $\tilde{o}_k$, and the other where the roles are swapped (i.e., the online network processes $\tilde{o}_k$):
\begin{equation}
\mathcal{L}_{\text{action}}
=
\frac{1}{2}
\mathbb{E}_{(\tilde{o}_q, \tilde{o}_k) \sim \mathcal{B}}
\left[
\mathcal{L}^a(\tilde{o}_q, \tilde{o}_k)
+
\mathcal{L}^a(\tilde{o}_k, \tilde{o}_q)
\right],
\label{loss:action}
\end{equation}
this symmetric objective encourages the prompt $\boldsymbol{p}^{u_i}$ to capture consistent action-relevant features invariant to temporal shifts and augmentation noise.

\textbf{(iii) Text Contrastive Learning.}
While the aforementioned visual and temporal action objectives address environmental and embodiment variations, they do not explicitly enforce the semantic alignment required for visuomotor tasks. To bridge the modality gap between visual observations and semantic goal specifications, we introduce a text contrastive learning objective. We define a distinct learnable text prompt $\boldsymbol{p}^{t}$, structurally identical to the visual prompts, which conditions the frozen CLIP encoder $\Phi$ to extract semantically grounded features.
Unlike domain-specific prompts that adapt to changing environments, $\boldsymbol{p}^{t}$ targets task-level semantic invariances orthogonal to domain shifts.

Specifically, for each navigation goal category $g \in \mathcal{C}$, we generate a text description using the template ``a photo of a [CLASS]'' and encode it via the frozen CLIP text encoder $\Psi$ to obtain the category-specific text embedding $\boldsymbol{t}_i^g = \Psi(\text{template}(g)) \in \mathbb{R}^{d_{\text{out}}}$.
These embeddings are precomputed and fixed to serve as semantic anchors.
Given a batch of observations $\mathcal{B} = \{(o_i, g_i)\}_{i=1}^{|\mathcal{B}|}$, we extract the text-prompted visual embedding $\boldsymbol{z}_i^{t} = \Phi(o_i, \boldsymbol{p}^{t})$. Furthermore, to enhance the adaptation of the learned representations, we introduce a semantic regularization mechanism. Specifically, we inject Gaussian noise into the visual embedding $\boldsymbol{z}_i^{t}$ to obtain a perturbed representation:
\begin{equation}
\tilde{\boldsymbol{z}}_i^{t} = \boldsymbol{z}_i^{t} + \boldsymbol{\zeta}, 
\quad 
\boldsymbol{\zeta} \sim \mathcal{N}(0, \sigma^2 \mathbf{I}) .
\end{equation}

We enforce that this perturbed representation maintains high semantic similarity with the corresponding text anchor $\boldsymbol{t}_{i}^g$. This strategy prevents feature collapse and ensures robustness against feature-level perturbations.
We then maximize the mutual information between the visual embedding and the corresponding text embedding following Eq.~\eqref{infonce}:
\begin{equation}
\mathcal{L}_{\text{NCE}}^{t}(\tilde{\boldsymbol{z}}_i^{t}, \boldsymbol{t}_{i}^g) = - \log \frac{\mathrm{sim}(\tilde{\boldsymbol{z}}_i^{t}, \boldsymbol{t}_{i}^g)}{\sum_{j=1}^{|\mathcal{B}|} \mathrm{sim}(\tilde{\boldsymbol{z}}_i^{t}, \boldsymbol{t}_j^g)}.
\end{equation}
Similar to visual contrastive learning, we also incorporate an MSE regularization term between matched image-text pairs:
\begin{equation}
\mathcal{L}_{\text{MSE}}^{t} = \frac{1}{|\mathcal{B}|} \sum_{i=1}^{|\mathcal{B}|} \|\tilde{\boldsymbol{z}}_i^{t} - \boldsymbol{t}_{i}^g\|_2^2.
\end{equation}
The final text objective is a weighted sum:
\begin{equation}
\mathcal{L}_{\text{text}} = \mathcal{L}_{\text{NCE}}^{t} + \lambda_{t} \mathcal{L}_{\text{MSE}}^{t},
\label{loss:text}
\end{equation}
where $\lambda_{t}$ balances the regularization strength. 
 By optimizing Eq.~\eqref{loss:text} alongside this regularization, $\boldsymbol{p}^{t}$ learns to steer the visual encoder to highlight goal-relevant semantics, providing a stable grounding signal that complements the adaptive domain-specific features.

Consequently, through the joint optimization of this hybrid contrastive learning strategy, we derive a set of domain-specific visual prompts, denoted as $\mathbf{P} = \{ \boldsymbol{p}^1, \boldsymbol{p}^2, \dots, \boldsymbol{p}^K \}$. Each prompt $\boldsymbol{p}^k$ is optimized to encapsulate distinct invariances regarding environmental or embodiment factors. This learned prompt pool is subsequently utilized in a frozen state to facilitate policy learning.

\subsection{Adaptive Prompt Orchestration}
\label{sec:policy learning}

\begin{figure*}[!t]
   \centering
   \includegraphics[width=1.0\textwidth]{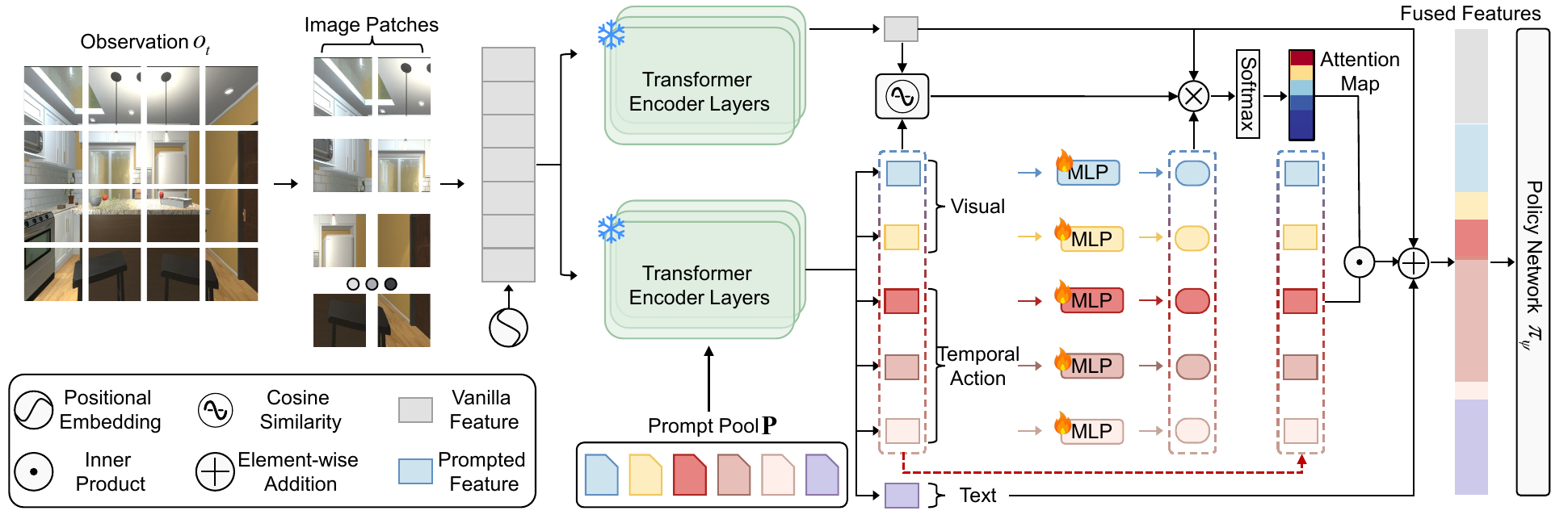} 
\caption{Architecture of the adaptive prompt orchestration mechanism. The frozen CLIP encoder extracts features from observation $o_t$ using the prompt pool $\mathbf{P}$. A dual-branch attention mechanism (combining learnable MLP projections and cosine similarity) dynamically re-weights the domain-specific features based on context. Crucially, the text-prompted feature bypasses this attention module and is fused directly. This design explicitly preserves goal-oriented semantics, which are invariant to domain factors, before feeding the policy network $\pi_{\psi}$.}
   \label{fig:network}
\end{figure*}

In this subsection, we detail the adaptive prompt orchestration phase of our approach, where the frozen domain-specific visual prompts $\mathbf{P} = \{ \boldsymbol{p}^1, \boldsymbol{p}^2, \dots, \boldsymbol{p}^K \}$ are adaptively composited and leveraged to train the visuomotor policy. We describe this process from two aspects: dual-branch attention mechanism and policy learning.

\textbf{Dual-Branch Attention Mechanism.}
While each prompt $\boldsymbol{p}^k$ excels at capturing invariances specific to a single domain factor, a robust embodied visuomotor policy requires simultaneous adaptability to multiple changing conditions. A naive aggregation of prompted features ignores the varying relevance of different domain factors for a given observation. To address this, we propose an adaptive prompt orchestration mechanism, denoted as $\mathcal{G}_{\text{attn}}$, that dynamically weights prompts based on semantic relevance, as depicted in Fig. \ref{fig:network}.

Given an RGB observation $o_t$ at timestep $t$, we first extract a vanilla embedding using the frozen CLIP encoder $\Phi$ without prompts, denoted as $\mathbf{z}_t^{v} = \Phi(o_t, \varnothing)$. Meanwhile, we obtain the set of domain-adapted embeddings $\{\mathbf{z}_t^k\}_{k=1}^K$ using the learned prompts, where $\mathbf{z}_t^k = \Phi(o_t, \boldsymbol{p}^k)$. Notably, the text-prompted embedding $\boldsymbol{z}_t^t = \Phi(o_t, \boldsymbol{p}^t)$ is excluded from the orchestration mechanism, as it encodes task-level semantic invariances. To fuse these representations, we employ a dual-guided attention mechanism driven by both learnable projections and semantic consistency. To fuse these representations, we employ a dual-guided attention mechanism driven by both learnable projections and semantic consistency.

On one hand, we compute a learnable attention score $s_k^{a}$ for each prompt by projecting the embedding into an attention key $k_k$:
\begin{equation}
    s_k^{a} = \frac{(\mathbf{z}_t^{v})^\top k_k}{\sqrt{d_{\text{out}}}}, \quad k_k = f_p(\mathbf{z}_t^k),
\end{equation}
where $f_p(\cdot)$ is a projection network consisting of two linear layers with a SiLU activation \cite{ramachandran2017searching}, and $d_{\text{out}}$ represents the output dimension of $\Phi$ with prompts.

On the other hand, to prevent the orchestration mechanism from focusing on irrelevant features, we introduce a regularization term based on the cosine similarity between the $\ell_2$-normalized vanilla and prompted embeddings to ensure semantic fidelity.
\begin{equation}
s_k^{c} = (\mathbf{z}_t^{v})^\top \mathbf{z}_t^{k} .
\end{equation}

The final attention weights $\alpha_k$ are derived via a multiplicative fusion of these scores, followed by a softmax normalization:
\begin{equation}
    \alpha_k = \frac{\exp(s_k^{a} \cdot s_k^{c})}{\sum_{j=1}^{K} \exp(s_j^{a} \cdot s_j^{c})}.
\end{equation}

The aggregated context representation $\mathbf{z}_t^f$ is computed via $\mathcal{G}_{\text{attn}}$ as a weighted sum of the prompted embeddings, enhanced by a residual connection to the vanilla feature to preserve original semantic information:
\begin{equation}
    \mathbf{z}_t^f = \mathbf{z}_t^{v} + \boldsymbol{z}_t^t + \mathcal{G}_{\text{attn} }(\{ \boldsymbol{z}^1_t, \dots, \boldsymbol{z}^K_t \}) = \mathbf{z}_t^{v} + \boldsymbol{z}_t^t + \sum_{k=1}^{K} \alpha_k \mathbf{z}_t^k.
\label{attention weights}
\end{equation}

This fused representation $\mathbf{z}_t^f$ adaptively integrates domain-specific invariances while maintaining task-relevant semantics, serving as a robust input for policy learning.

\begin{algorithm}[t]
\caption{Training Pipeline of CAPO}
\label{alg:training}
\KwIn{Trajectory dataset $\mathcal{D}$, frozen CLIP encoder $\Phi$, learnable prompt pool $\mathbf{P}$}
\KwOut{Optimized navigation policy $\pi_{\psi}$}

\BlankLine
\textbf{// Contrastive Prompt Learning} \\
\ForEach{iteration}{
    Sample minibatch $\mathcal{B} \sim \mathcal{D}$ \\
    \textbf{// Visual Contrastive Learning:} \\
    Generate augmented views $o_i, o_i^+$ for $o_i \in \mathcal{B}$ \\
    Compute $\mathcal{L}_{\text{visual}}$ via Eq.~\eqref{loss:visual} \\
    $\mathbf{P} \leftarrow \mathbf{P} - \eta_v \nabla_{\mathbf{P}} \mathcal{L}_{\text{visual}}$ \\
    \textbf{// Temporal Action Contrastive Learning:} \\
    Sample $(o_q, o_k)$ from $\mathcal{B}$ \\
    Generate augmented views $(\tilde{o}_q, \tilde{o}_k)$ \\
    Compute  $\mathbf{z}_{\text{on}}^q = f_{\omega}(\rho_{\omega}(\Phi(\tilde{o}_q, \boldsymbol{p}^{u_i})))$\\
    Compute  $\mathbf{z}_{\text{tgt}}^k = \texttt{SG}(h_{\nu}(\Phi(\tilde{o}_k, \boldsymbol{p}^{u_i})))$ \\
    Compute $\mathcal{L}_{\text{action}}$ via Eq.~\eqref{loss:action} \\
    $\mathbf{P} \leftarrow \mathbf{P} - \eta_a \nabla_{\mathbf{P}} \mathcal{L}_{\text{action}}$, $\omega \leftarrow \omega - \eta_a \nabla_{\omega} \mathcal{L}_{\text{action}}$ \\
    $\nu \leftarrow \beta \nu + (1 - \beta) \omega$ \\
    \textbf{// Text Contrastive Learning:} \\
    Sample $(o_i, g_i)$ from $\mathcal{B}$ \\
    Compute $\mathcal{L}_{\text{text}}$ via Eq.~\eqref{loss:text}\\
    $\mathbf{P} \leftarrow \mathbf{P} - \eta_t \nabla_{\mathbf{P}} \mathcal{L}_{\text{text}}$ \\
}

\BlankLine
\textbf{// Adaptive Prompt Orchestration} \\
Freeze $\mathbf{P}$ and $\Phi$, initialize $\mathcal{G}_{\text{attn}}$ and $\pi_\psi$\\

\ForEach{iteration}{
    Sample rollout trajectory $\tau = \{o_t, a_t, r_t, g\}_{t=1}^T$ \\
    Compute $\boldsymbol{z}_t^v = \Phi(o_t, \varnothing)$, $\boldsymbol{z}_t^t = \Phi(o_t, \boldsymbol{p}^t)$, $\{\boldsymbol{z}_t^k\}_{k=1}^K$ \\
    Compute $\boldsymbol{z}_t^f$ via Eq.~\eqref{attention weights} \\
    Compute $\hat{o}_t = [\boldsymbol{z}_t^f, \mathbf{g}, \mathbf{e}_{a_{t-1}}]$, $h_t = \text{GRU}(\hat{o}_t, h_{t-1})$ \\
    Output action $a_t \sim \pi_\psi(\cdot|h_t)$ \\
    Compute PPO objectives $\mathcal{L}_{\text{PPO}}$ via Eq.~\eqref{eq:ppo_actor} \\
    $\psi \leftarrow \psi - \eta_{\psi} \nabla_{\psi} \mathcal{L}_{\text{PPO}}$ \\
    }
\end{algorithm}

\textbf{Policy Learning.}
Given the fused representation $\mathbf{z}_t^f$, we construct the comprehensive observation embedding $o_t$ by incorporating task-specific semantic and history information. Specifically, the target object category $g \in \mathcal{C}$ is encoded as a one-hot vector $\mathbf{g}$, and the action taken at the previous timestep $a_{t-1} \in \mathcal{A}$ is mapped to a learnable embedding $\mathbf{e}_{a_{t-1}}$. These components are concatenated with the visual feature to form the policy input: $\hat{o}_t = [\mathbf{z}_t^f, \mathbf{g}, \mathbf{e}_{a_{t-1}}]$.

To mitigate the partial observability inherent in embodied navigation, we employ a Gated Recurrent Unit (GRU)~\cite{turkoglu2021gating} to model temporal dependencies. The GRU processes the current observation embedding $\hat{o}_t$ alongside the previous hidden state $h_{t-1}$ to generate the current hidden state $h_t$, which serves as the shared representation for the Actor-Critic architecture. The actor head predicts the action distribution $\pi_\psi(a_t|h_t)$, while the critic head estimates the state value $V(h_t)$. We optimize the policy $\pi_\psi$ using the Proximal Policy Optimization (PPO) algorithm~\cite{schulman2017proximal}, which trains the entire network architecture (including the GRU and embeddings) with the objective as:
\begin{equation}
\label{eq:ppo_actor}
\mathcal{L}_{\text{PPO}}
=
-\mathbb{E}_t \left[
\min \left(
r_t \hat{A}_t,\ 
\text{clip}(r_t, 1 - \epsilon, 1 + \epsilon)\hat{A}_t
\right)
\right],
\end{equation}
where $r_t$ denotes the probability ratio between the current policy and the previous policy, $\hat{A}_t$ is the advantage estimate at time step $t$, and $\epsilon$ is the clipping hyperparameter.

\subsection{Summary and Theoretical Justification}
\label{text:algorithm}

\textbf{Algorithm Summary.} We summarize the overall training pipeline of CAPO in Algorithm \ref{alg:training}. In the contrastive prompt learning, the agent learns a diverse set of domain-specific visual prompts through a hybrid contrastive learning strategy, capturing invariances across visual, temporal, and semantic dimensions. In the adaptive prompt orchestration, these learned prompts are frozen, and a visuomotor policy is optimized using the adaptive prompt orchestration mechanism, enabling dynamic adaptation to unseen domain factors.

\textbf{Theoretical Justification.} 
The core objective of CAPO is to construct a visuomotor policy $\pi_\psi$ that minimizes the expected risk across a distribution of domain factors $\mathcal{F}$. We posit that for any context $c \in \mathcal{F}$, there exists an optimal oracle representation $\boldsymbol{z}_c^\star$. The goal is to approximate this oracle via the adaptive prompt orchestration mechanism $\mathcal{G}_{\text{attn}}$.

Let $R_E(\pi) = \mathbb{E}_{c}[\mathcal{L}_{\text{task}}(\pi(\boldsymbol{z}^f(c)), y_c^\star)]$ be the expected risk, where $\mathcal{L}_{\text{task}}$ denotes the task-specific loss function (e.g., PPO objective) and $y_c^\star$ represents the optimal action or value for the given context $c$. Let $R^\star$ be the minimum risk achievable by the oracle. We define the \emph{excess risk} as the difference between the learned policy's performance and the oracle's performance.

\begin{theorem}\label{theorem1}
Let $\mathcal{H}_{\mathbf{P}}$ be the hypothesis space spanned by the convex hull of the learned prompt pool $\mathbf{P}$. Assuming the task loss is $L_{\ell}$-Lipschitz and the policy is $\lambda_\pi$-Lipschitz, with probability at least $1 - \delta$, the excess risk is bounded by:
\begin{equation}
\begin{split}
R_E(\hat{\pi}) - R^\star \leq &\ \tilde{O}\left(\sqrt{\frac{1}{N}}\right) \\
&\hspace{-1.4em}+ L_{\ell}\lambda_\pi \cdot \mathbb{E}_c \left[ \min_{\alpha \in \Delta_K} \left\lVert \sum_{k=1}^K \alpha_k \boldsymbol{z}^k - (\boldsymbol{z}_c^\star - \boldsymbol{z}^v) \right\rVert_2 \right].
\end{split}
\end{equation}
\end{theorem}

\begin{proof}
We decompose the excess risk into estimation error and approximation error, and bound them simultaneously \cite{feng2023maskcon}. Assuming the loss is $L_{\ell}$-Lipschitz and $\pi_\psi$ is $\lambda_\pi$-Lipschitz continuous \cite{lei2023generalization}, we have:
\begin{equation}
\begin{aligned}
R_E(\hat{\pi}) - R^\star &= \underbrace{(R_E(\hat{\pi}) - \inf_{\pi \in \mathcal{H}_{\mathbf{P}}} R_E(\pi))}_{\text{Estimation Error}} \\
&\quad + \underbrace{(\inf_{\pi \in \mathcal{H}_{\mathbf{P}}} R_E(\pi) - R^\star)}_{\text{Approximation Error}} \\
&\leq \underbrace{\left( 2\mathcal{R}(\Pi) + \sqrt{\frac{\log(1/\delta)}{2N}} \right)}_{\text{Bound via Uniform Convergence}} \\
&\quad + \underbrace{L_{\ell}\lambda_\pi \cdot \mathbb{E}_c \left[ \min_{\alpha} \lVert \boldsymbol{z}^f - \boldsymbol{z}_c^\star \rVert_2 \right]}_{\text{Bound via Lipschitz Continuity}}.
\end{aligned}
\end{equation}
Substituting the definition of the fused representation $\boldsymbol{z}^f = \boldsymbol{z}^v + \sum \alpha_k \boldsymbol{z}^k$ into the second term yields the stated bound. See \textbf{Supplementary Material} for the detailed derivation.
\end{proof}


\textbf{Remark.} Theorem~\ref{theorem1} decomposes the excess risk into estimation and approximation terms. The bound implies that the risk is controlled by the feature distance to the oracle (scaled by $L_{\ell}\lambda_\pi$). This justifies our design: hybrid contrastive learning ensures the prompt pool $\mathbf{P}$ is sufficiently diverse to reduce the approximation gap, while the attention mechanism $\mathcal{G}_{\text{attn}}$ dynamically optimizes $\alpha$ for precise alignment.

\section{Experimental Setup}
This section presents the experimental setup of CAPO, covering the simulation environments and domain configurations, evaluation metrics and baselines, as well as network architectures and hyperparameter settings.

\subsection{Simulation Setup}
We conduct experiments in the AI2-THOR~\cite{kolve2017ai2} environment. We specifically utilize \texttt{FloorPlan21} for both training and evaluation on the \texttt{ObjectNav} task, where the data splits are distinguished by different combinations of domain factors. Success is defined as reaching within 1.0 meters of one of 12 target categories within a maximum horizon of $T_{\max} = 500$. For contrastive prompt learning, we collect an expert dataset covering diverse domain configurations defined by $\mathcal{F} = (\phi, \psi, \theta, \delta, \ell^b, \ell^c, \ell^s, \ell^h)$. We treat lighting parameters $\ell^b, \ell^c, \ell^s, \ell^h$ (brightness, contrast, saturation, hue) as continuous variables, while rotation $\psi$, look angle $\theta$, and step size $\delta$ are discrete. The FOV $\phi$ is categorized into three groups to mitigate training instability, resulting in a total of 10 distinct domain factors. These domain factors are randomly sampled from a uniform distribution to construct the source domain, seen target domain, and unseen target domain settings for training and evaluation. The collected dataset comprises 1,468 expert trajectories with 22,754 samples, where cross-domain trajectory pairs are aligned based on a semantic similarity threshold $\kappa \ge 0.7$. The agent operates with a discrete action space consisting of move ahead, rotate left, rotate right, look up, look down, and end.

\subsection{Metrics}
To quantitatively evaluate our approach, we report results on several standard metrics widely adopted \cite{kolve2017ai2}:
\begin{itemize}
    \item \textit{Success Rate (SR) (\%)}: The percentage of evaluation episodes where the agent successfully navigates to within $1.0\,\text{m}$ of the target and executes the termination action.
    \item \textit{Success weighted by Path Length (SPL)}: A composite metric balancing task completion with trajectory efficiency, defined as
    $\text{SPL} = \frac{1}{N} \sum_{i=1}^{N} \mathbb{I}_i \cdot \frac{d_i^\star}{\max(d_i^\star, p_i)}$,
    where $\mathbb{I}_i \in \{0,1\}$ indicates whether episode $i$ is successful,
    $d_i^\star$ denotes the optimal geodesic distance to the goal, and $p_i$ is the actual path length traversed by the agent.
    \item \textit{Navigation Error (NE) $(\mathrm{m})$}: The average Euclidean distance between the agent's final position and the nearest goal object upon episode termination.
    \item \textit{Episode Length (EL)}: The average number of time steps consumed per episode.
\end{itemize}

\subsection{Baselines}
To comprehensively validate the effectiveness of CAPO, we compare against five representative baselines:

\begin{itemize}
    \item CURL \cite{laskin2020curl}: A representative visual contrastive learning approach that learns state representations by constructing positive and negative pairs through data augmentation on the same observation.

    \item ACO \cite{zhang2022learning}: An action-conditioned contrastive learning baseline that treats observations associated with the same executed action as positive pairs, while regarding others as negatives to capture action-relevant features.

    \item ATC \cite{stooke2021decoupling}: A temporal contrastive learning approach that encourages temporal consistency by treating adjacent frames as positive pairs and distant frames as negatives.

    \item ConPE \cite{choi2024efficient}: A state-of-the-art baseline that utilizes CLIP to facilitate both visual and action-conditioned contrastive learning, while employing a unimodal attention mechanism for policy adaptation.

    \item PPO \cite{schulman2017proximal}: A standard RL baseline that modifies the vanilla PPO algorithm by replacing the visual backbone with a frozen CLIP encoder for a fair comparison.
\end{itemize}

\subsection{Parameter and Architecture Configuration}

\textbf{Contrastive Prompt Learning.} 
We utilize the CLIP ViT-B/32 model as the frozen backbone $\Phi$, which takes RGB images of size $3 \times 224 \times 224$ as input. The learnable prompt pool $\mathbf{P}$ consists of $K=10$ domain-specific prompts (4 for appearance, 5 for action attributes, and 1 for text). Each prompt $\boldsymbol{p}^k$ has a length of $L=8$ with an embedding dimension of $D=768$, initialized via Xavier uniform initialization. 

For visual contrastive learning, we use a batch size of $|\mathcal{B}|=256$ and train for 500 epochs. The loss scaling hyperparameters $\lambda_v$ and $\lambda_t$ are set to 1.0. We use an SGD optimizer with momentum 0.9, weight decay $10^{-3}$, and a polynomial decay schedule with an initial learning rate of 0.001. For temporal action contrastive learning, we employ a batch size of 64 and train for 500 epochs. The online network comprises the prompted encoder $\Phi$, followed by a two-layer MLP projector $\rho_{\omega}$ ($d_{\text{out}} = 512 \to 4096 \to 16$) and a two-layer MLP predictor $f_{\omega}$ ($16 \to 512 \to 16$), both utilizing ReLU activations and batch normalization. The target network is updated via an exponential moving average of the online network with a momentum coefficient of $\beta=0.99$. The optimizer settings follow the same schedule as visual contrastive learning. For text contrastive learning, we employ a batch size of 256 and train for 500 epochs. We use an SGD optimizer with Nesterov momentum 0.9, weight decay $10^{-3}$, and a polynomial decay schedule with an initial learning rate of 0.002.

\textbf{Adaptive Prompt Orchestration.}
Here, the prompt pool $\mathbf{P}$ and encoder $\Phi$ are frozen. 
The attention mechanism $\mathcal{G}_{\text{attn}}$ utilizes a projection network $f_p$ to compute attention keys, which is implemented as a two-layer MLP ($512 \to 128 \to 512$) with SiLU activation. The policy network $\pi_\psi$ receives a concatenated input of the fused visual feature $\mathbf{z}_t^f \in \mathbb{R}^{512}$ and the one-hot goal vector $\mathbf{g} \in \mathbb{R}^{12}$. The temporal dependencies are modeled by a single-layer GRU with a hidden state size of 512. The actor and critic heads are linear layers mapping the GRU output to action logits and value estimates, respectively.

We optimize the policy using PPO with 48 parallel environments. The training spans $3 \times 10^6$ environment steps. In each update cycle, we collect a rollout buffer of 4,800 transitions (rollout length of 100 per environment). 
Key hyperparameters are as follows: the learning rate is set to $3 \times 10^{-4}$ and linearly decayed to $10^{-5}$; the PPO clipping range is $\epsilon=0.2$, and the discount factor is $\gamma=0.99$. We set the GAE parameter to 0.95 and the value loss coefficient to 0.5. Regarding the reward function design, we employ a composite reward function. Specifically, the agent receives +10.0 for successfully reaching the target object and -0.01 per time step to promote efficiency. Additionally, distance-based reward shaping is applied: the agent receives $r = d_{t-1} - d_t$ at each step, where $d_t$ denotes the geodesic distance to the target, with the magnitude clamped to the actual distance traveled to prevent reward exploitation. All training and evaluation experiments are conducted on a single NVIDIA GeForce RTX 4090 GPU.

\section{Results and Analysis}
This section presents the experimental results and key analyses validating the effectiveness of CAPO. We evaluate its performance and adaptability across diverse domain factors and embodiment configurations in simulation, benchmarking it against various baselines. Furthermore, we conduct ablation studies to examine the contribution of each core component.

\subsection{Comparison with Baselines}

The training process is conducted on the 4 source domains, each characterized by a unique combination of domain factors. For evaluation, we utilize a comprehensive set comprising the 4 source domains, 30 seen target domains, and 10 unseen target domains to rigorously assess adaptation. The primary distinction lies in whether the domain configurations were encountered during the representation learning phase.

The reward curves of the training process are illustrated in Fig. \ref{fig:reward}. As illustrated, our approach (CAPO) achieves the fastest convergence and the strongest asymptotic performance, consistently outperforming all baselines. This significant advantage stems from three pivotal design choices: (i) a synergized training paradigm, where the prompt basis is frozen while the orchestration weights and the policy are jointly optimized, effectively combining the stability of decoupled representations with the flexibility of end-to-end learning. This hybrid paradigm bypasses the sample inefficiency typically associated with end-to-end methods, while simultaneously overcoming the rigidity of static representations induced by fully decoupled pipelines; (ii) text contrastive learning, unlike baselines relying solely on visual features, CAPO leverages the rich semantic priors embedded in the CLIP model. By incorporating learnable text prompts, we explicitly align visual representations with goal-oriented semantics. This alignment not only enhances sample efficiency but also provides a more robust feature space, allowing the policy to converge to a higher optimal performance; (iii) adaptive prompt orchestration, instead of utilizing static feature representations, we introduce a dynamic attention mechanism that adaptively aggregates domain-specific prompts based on the current context. This enables the policy to rapidly identify and upweight the most relevant features, effectively accelerating the learning process and achieving better performance.

\begin{figure}[t!]
    \centering
    \includegraphics[width=1.0\linewidth]{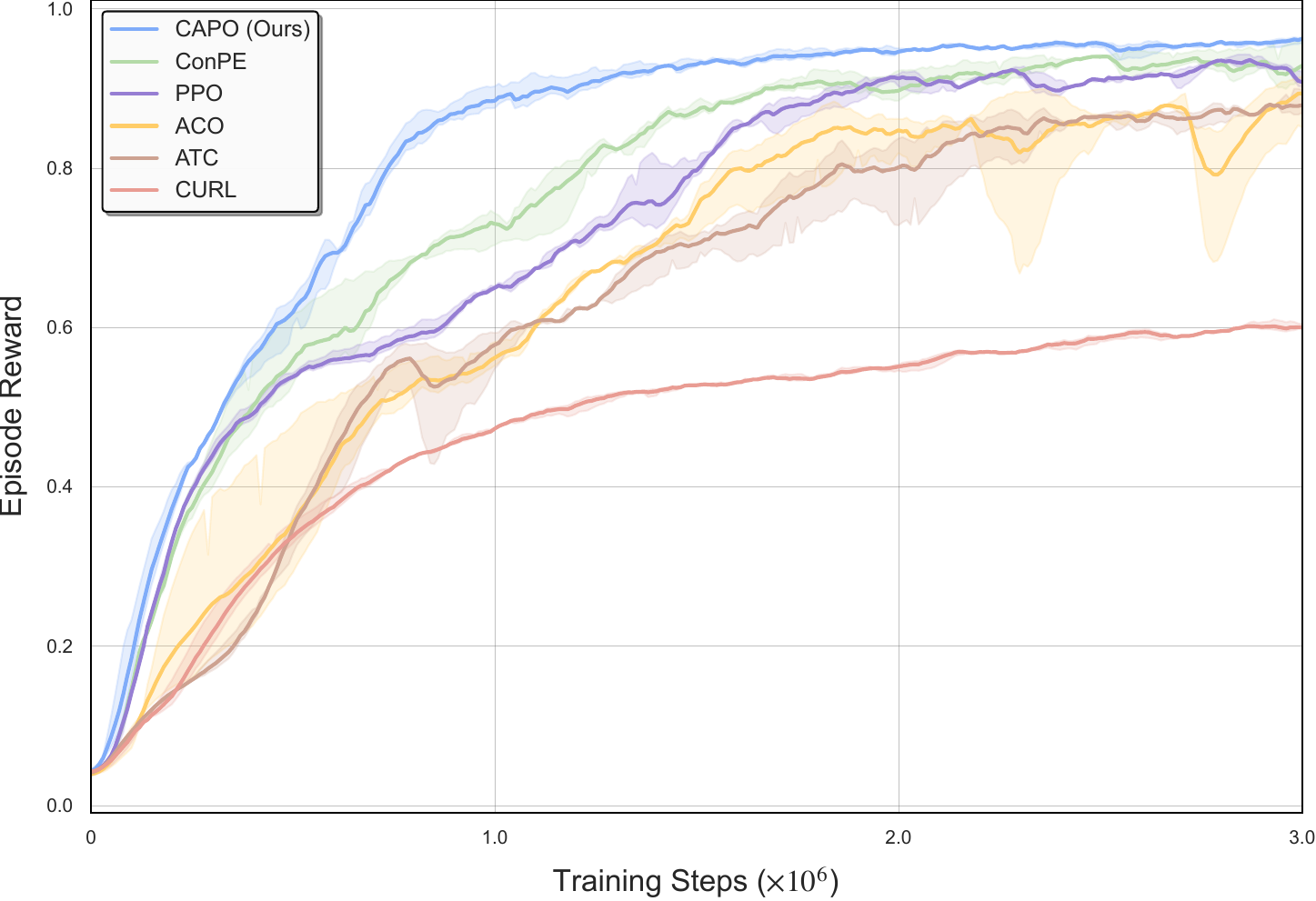}
    \caption{Training curves of episodic reward for all baselines. All results are averaged over three random seeds, with shaded regions indicating confidence intervals. CAPO demonstrates superior sample efficiency and asymptotic performance, converging faster and achieving higher rewards than all baselines.}
    \label{fig:reward}
\end{figure}

Regarding the baselines, ConPE achieves the second-best performance. However, it falls short of CAPO due to its reliance on unimodal representation learning and the lack of explicit cross-modal semantic grounding, which limits its ability to fully capture goal-relevant cues. Notably, PPO, which simply utilizes a frozen CLIP encoder, outperforms some contrastive learning baselines. We attribute this to the robust adaptation capability of the pre-trained CLIP encoder, which provides superior initial features compared to models trained from scratch. Conversely, while ATC and ACO attempt to capture temporal or action-relevant features, they converge more slowly and underperform PPO. This counter-intuitive result indicates that their coarse-grained modeling and the absence of semantic guidance fail to capture the fine-grained environmental dynamics essential for precise navigation. CURL exhibits the weakest performance, which suggests that performing instance-level contrastive learning on limited interaction data leads to catastrophic forgetting of pre-trained general knowledge and suffers significantly from distribution shift during fine-tuning.

\begin{table*}[t!]
\centering
\newcolumntype{Y}{>{\centering\arraybackslash}X}

\setlength{\tabcolsep}{1pt} 
\renewcommand{\arraystretch}{1.2}

\begin{threeparttable}
\caption{Performance of Baselines in Simulation Environments}
\label{tab:baseline_eval}

\footnotesize 

\begin{tabularx}{\textwidth}{l *{12}{Y}}
\toprule
\multirow{2}{*}{Approach} &
\multicolumn{4}{c}{Source Domains} &
\multicolumn{4}{c}{Seen Target Domains} &
\multicolumn{4}{c}{Unseen Target Domains} \\
\cmidrule(lr){2-5} \cmidrule(lr){6-9} \cmidrule(lr){10-13}
& SR$\uparrow$ & SPL$\uparrow$ & NE$\downarrow$ & EL$\downarrow$
& SR$\uparrow$ & SPL$\uparrow$ & NE$\downarrow$ & EL$\downarrow$
& SR$\uparrow$ & SPL$\uparrow$ & NE$\downarrow$ & EL$\downarrow$ \\
\midrule

CURL~\cite{laskin2020curl}
& 52.0$\pm$1.9 & 0.32$\pm$0.07 & 0.48$\pm$0.08 & 32$\pm$6
& 15.5$\pm$7.0 & 0.12$\pm$0.07 & 0.93$\pm$0.15 & 55$\pm$13
& 10.3$\pm$4.5 & 0.07$\pm$0.03 & 1.06$\pm$0.05 & 70$\pm$11 \\

ACO~\cite{zhang2022learning}
& 79.1$\pm$22.8 & 0.51$\pm$0.16 & 0.15$\pm$0.05 & 40$\pm$21
& 59.5$\pm$17.1 & 0.37$\pm$0.12 & 0.17$\pm$0.07 & 49$\pm$11
& 42.6$\pm$28.1 & 0.27$\pm$0.20 & 0.17$\pm$0.11 & 53$\pm$9 \\

ATC~\cite{stooke2021decoupling}
& 82.1$\pm$8.1 & 0.53$\pm$0.10 & 0.12$\pm$0.04 & 27$\pm$6
& 70.4$\pm$8.8 & 0.44$\pm$0.08 & 0.15$\pm$0.09 & 41$\pm$16
& 51.1$\pm$11.2 & 0.31$\pm$0.09 & 0.15$\pm$0.08 & 46$\pm$11 \\

ConPE~\cite{choi2024efficient}
& 96.1$\pm$2.3 & 0.64$\pm$0.06 & 0.03$\pm$0.01 & 20$\pm$4
& 83.3$\pm$3.9  & 0.53$\pm$0.09 & 0.06$\pm$0.08 & 33$\pm$6
& 80.5$\pm$6.2  & 0.52$\pm$0.09 & 0.07$\pm$0.07 & 36$\pm$4 \\

PPO~\cite{schulman2017proximal}
& 86.5$\pm$5.9 & 0.53$\pm$0.10 & 0.17$\pm$0.08 & 25$\pm$7
& 77.7$\pm$3.6 & 0.48$\pm$0.09 & 0.15$\pm$0.07 & 38$\pm$16
& 59.8$\pm$8.1 & 0.35$\pm$0.10 & 0.19$\pm$0.06 & 43$\pm$8 \\

\rowcolor{gray!15}
\textbf{CAPO (Ours)}
& \textbf{97.9$\pm$1.2} & \textbf{0.66$\pm$0.04} & \textbf{0.02$\pm$0.01} & \textbf{18$\pm$3}
& \textbf{90.9$\pm$3.4} & \textbf{0.61$\pm$0.05} & \textbf{0.04$\pm$0.06} & \textbf{29$\pm$5}
& \textbf{86.4$\pm$5.7} & \textbf{0.54$\pm$0.06} & \textbf{0.06$\pm$0.07} & \textbf{32$\pm$3} \\

\bottomrule
\end{tabularx}
\end{threeparttable}
\end{table*}

\begin{figure*}[!t]
   \centering
   \includegraphics[width=1.0\textwidth]{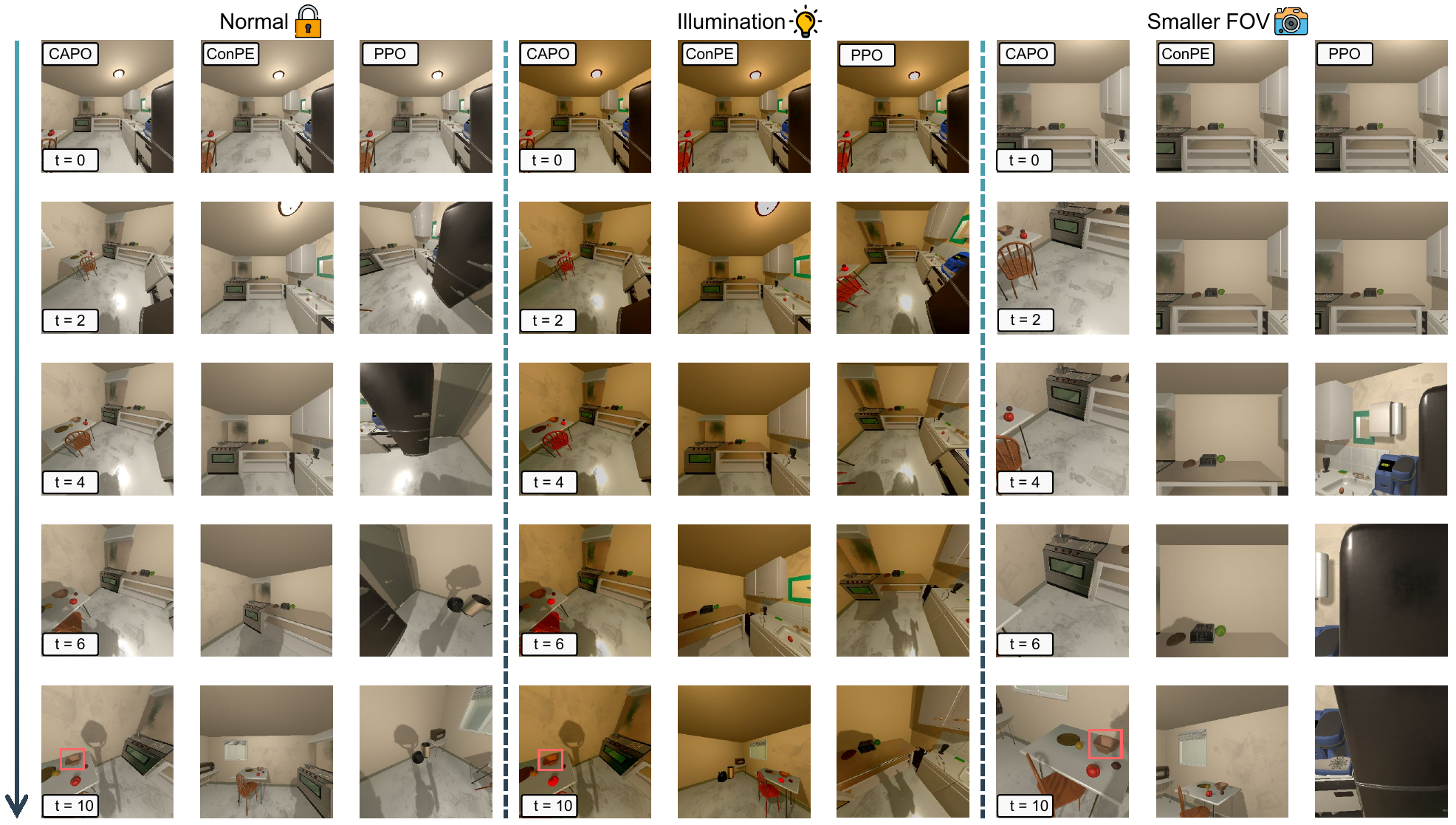} 
\caption{Visualization of navigation episodes for the goal ``Bowl'' under diverse domain shifts. The columns display egocentric frame sequences from CAPO, ConPE, and PPO across three scenarios: Normal, Illumination Change, and Cross-Embodiment (Smaller FOV). CAPO rapidly identifies and approaches the target in all conditions, demonstrating robust zero-shot adaptation. ConPE exhibits partial adaptability but suffers from delayed target lock-on, requiring more steps to complete the task. In contrast, PPO fails to adapt to environmental interference, resulting in disorientation and task failure.}
   \label{fig:tra}
\end{figure*}

\begin{figure}[htbp]
    \centering
    
    \begin{subfigure}{\linewidth}
        \centering
        \includegraphics[width=0.8\linewidth]{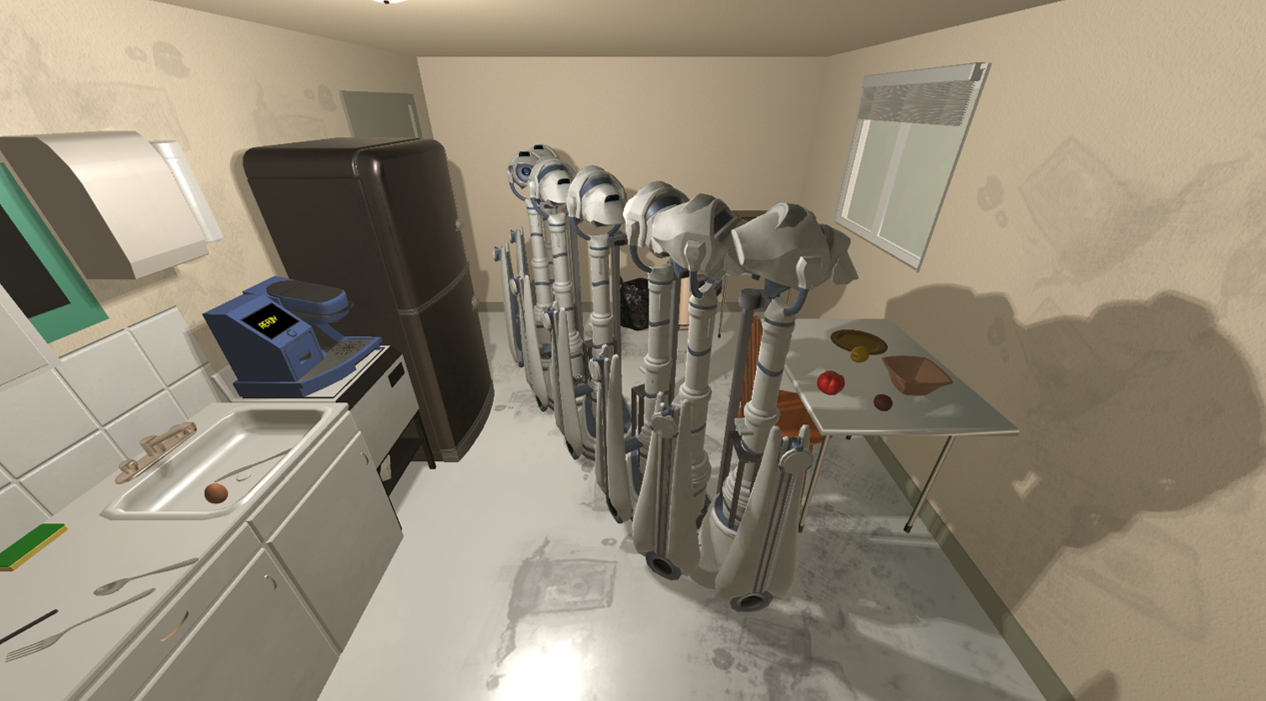}
        \caption{ManipulaTHOR}
        \label{fig:a}
    \end{subfigure}
    
    \begin{subfigure}{\linewidth}
        \centering
        \includegraphics[width=0.8\linewidth]{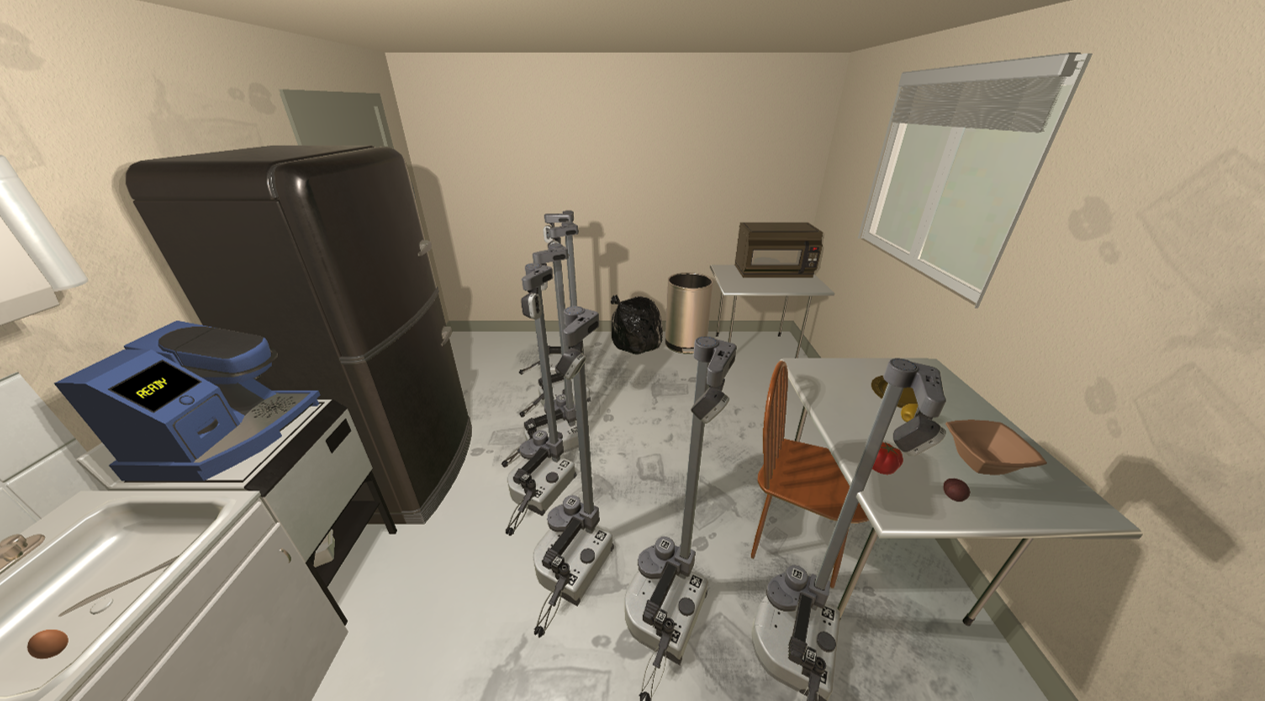}
        \caption{Stretch RE1}
        \label{fig:b}
    \end{subfigure}
    
    \begin{subfigure}{\linewidth}
        \centering
        \includegraphics[width=0.8\linewidth]{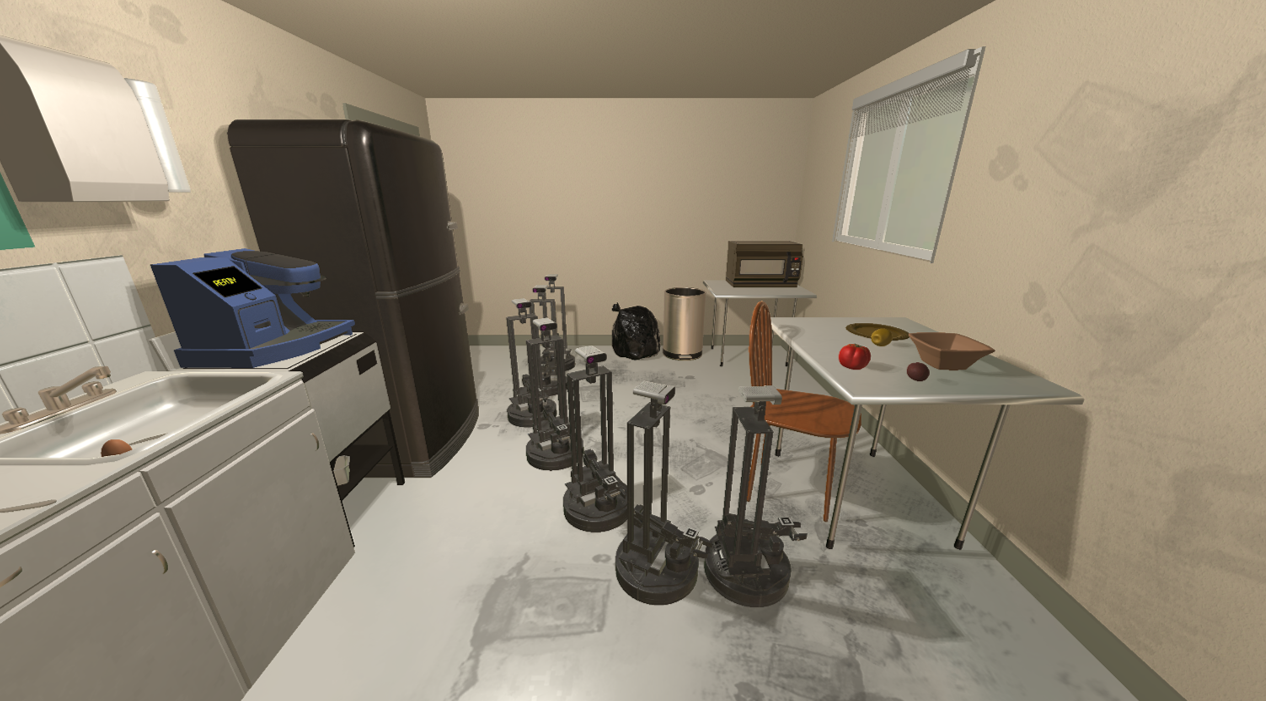}
        \caption{LoCoBot}
        \label{fig:c}
    \end{subfigure}
    
    \caption{Zero-shot policy adaptation across distinct unseen robot embodiments. The policy, trained on a standard agent (a) ManipulaTHOR, is directly deployed onto (b) Stretch RE1, and (c) LoCoBot. Despite significant discrepancies in embodiment configurations compared to the source domain, CAPO successfully adapts to these heterogeneous embodiments, enabling effective navigation without any fine-tuning.}

    \label{fig:cross}
\end{figure}

Quantitative results across all domains are summarized in Table \ref{tab:baseline_eval}. The results demonstrate that CAPO maintains consistently strong adaptation capabilities across all domain factors. In contrast, most baselines exhibit significant performance degradation, particularly in unseen domains, indicating a failure to handle distribution shifts effectively. While ConPE demonstrates a certain level of adaptability, it still falls short of CAPO. This performance gap further highlights the limitation of relying solely on unimodal visual cues. By integrating text contrastive learning with a dynamic adaptive prompt orchestration, CAPO effectively bridges the gap between visual features and goal semantics while selectively aggregating domain-specific features. This multi-modal alignment allows the policy to maintain robustness against visual distractions and adapt flexibly to unseen embodiment configurations, whereas baselines constrained by fixed or unimodal representations fail to generalize.

\begin{figure}[t!]
    \centering
    \includegraphics[width=0.9\linewidth]{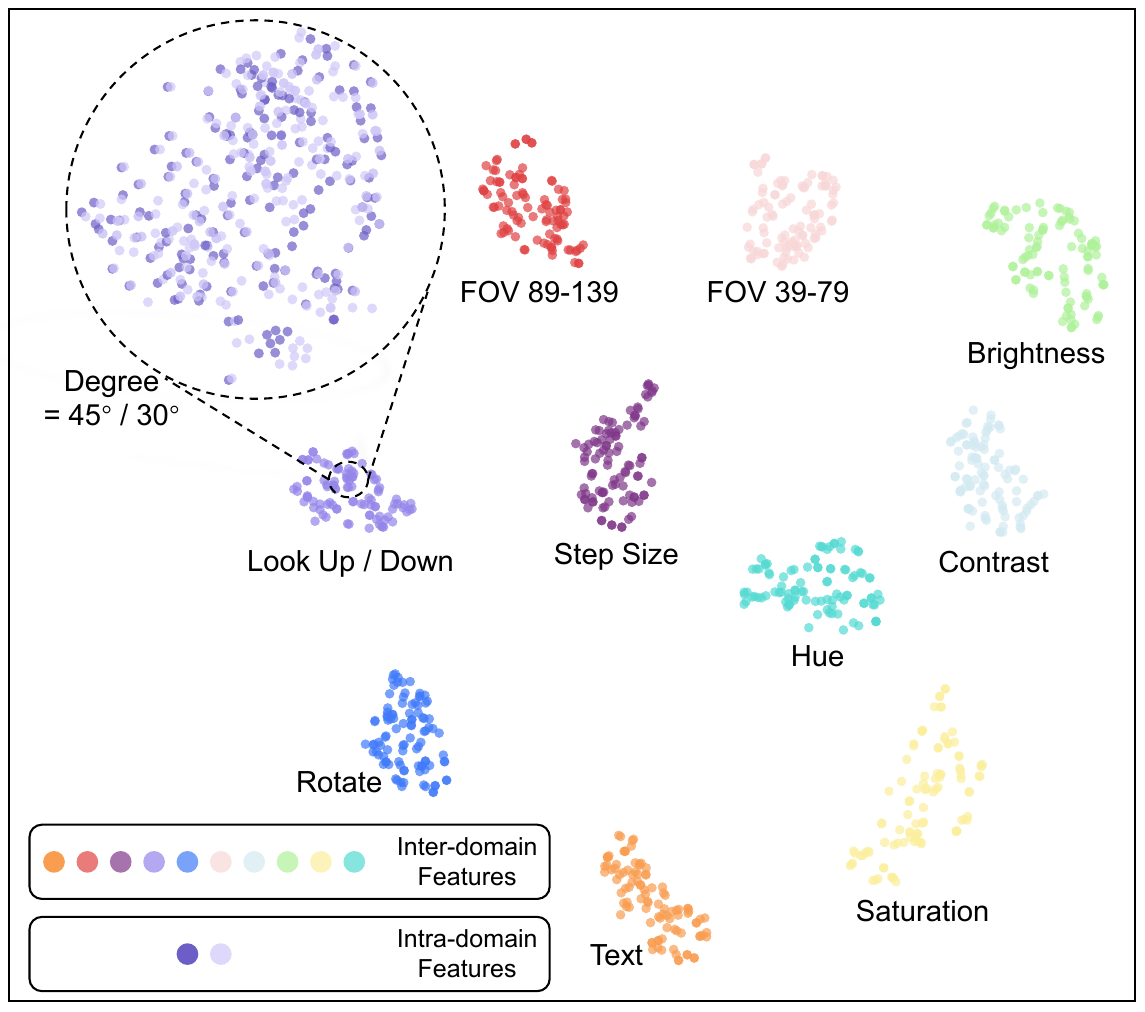}
    \caption{t-SNE visualization of learned prompt representations. The plot reveals that prompts across different domains are well-separated, while features within the same domain form tight clusters. This demonstrates that CAPO effectively disentangles visual and embodiment features.}
    \label{fig:tsne}
\end{figure}

\begin{figure}[t!]
    \centering
    \includegraphics[width=1.0\linewidth]{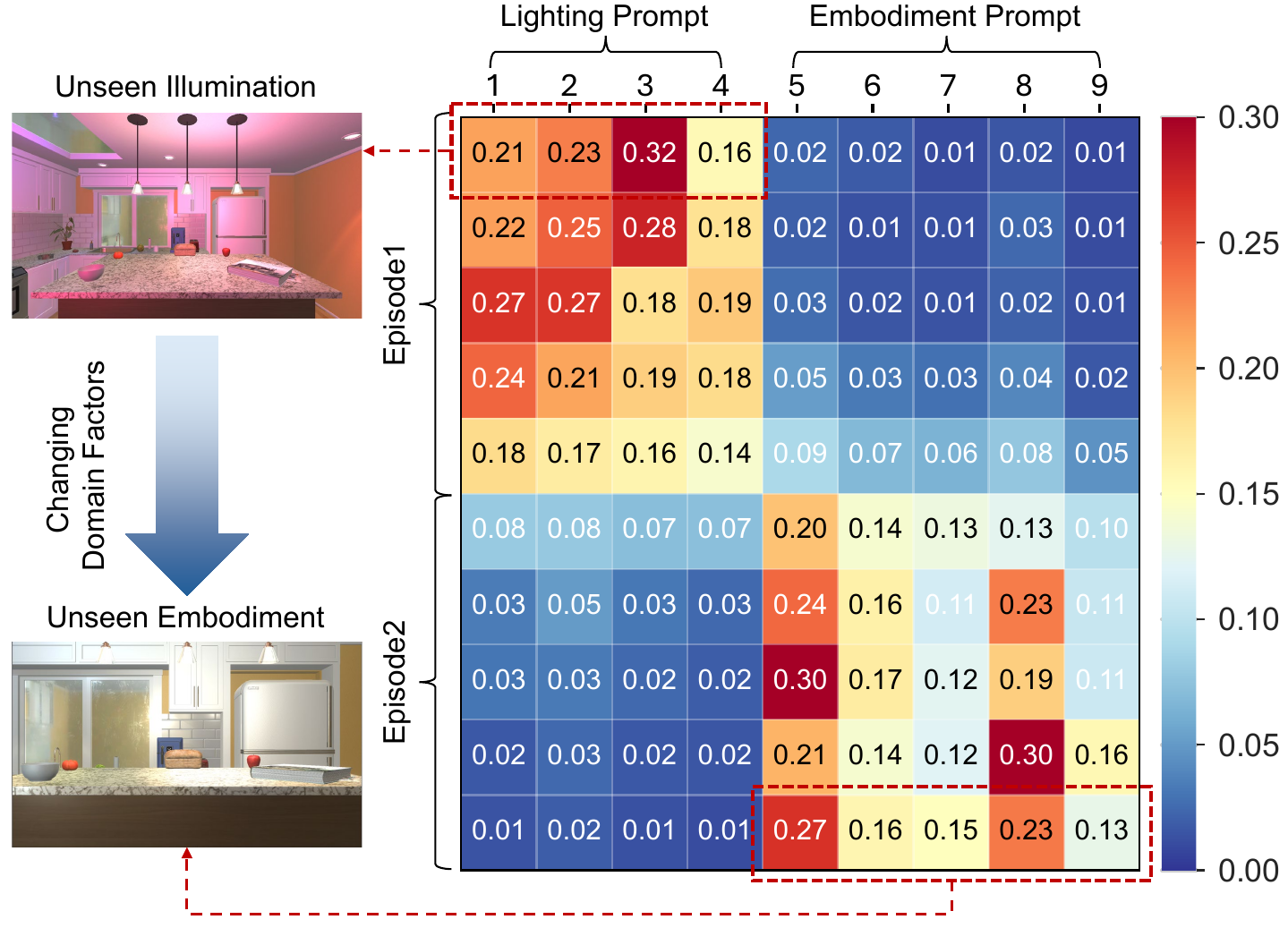}
    \caption{Visualization of attention weights during different navigation episodes. Note that the text prompt is excluded here as it is fused directly. The heatmap shows dynamic shifts in attention distribution over time, indicating that the agent adaptively re-weights prompts to prioritize relevant features as the domain factor changes. This validates the effectiveness of our adaptive prompt orchestration mechanism during deployment.}
    \label{fig:attention}
\end{figure}

\subsection{Cross-Embodiment Policy Adaptation}

Egocentric sequences of the agent under different domain factors and embodiment configuration are visualized in Fig. \ref{fig:tra}. It is shown that CAPO exhibits the most decisive navigation behaviors across evaluation domains. Specifically, our agent maintains stable visual focus on the target ``Bowl'' despite drastic lighting or FOV changes, whereas baselines often perform redundant actions or lose tracking. This suggests the effectiveness of our text contrastive learning and multi-modal prompt orchestration mechanism. By explicitly grounding visual features in invariant language semantics and dynamically filtering domain noise, CAPO ensures the policy remains robustly aligned with the task goal, avoiding the degradation seen in approaches lacking this cross-modal grounding.

Furthermore, to rigorously validate the effectiveness of CAPO beyond sensor parameter shifts (e.g., FOV), we extended our evaluation to physical cross-embodiment adaptation. As illustrated in Fig. \ref{fig:cross}, we directly deployed the policy trained solely on the ManipulaTHOR source agent onto two heterogeneous robot platforms: Stretch RE1 and LoCoBot. These embodiments introduce drastic variations in camera mounting heights and kinematic constraints, which typically disrupt the state distribution for conventional agents. However, CAPO exhibits remarkable adaptability in this challenging setting. By leveraging the learned prompts, the orchestration mechanism effectively compensates for the discrepancies in egocentric views caused by different physical configurations. This allows the agent to execute coherent navigation actions across diverse embodiments, further verifying CAPO's potential for cross-embodiment deployment.

To further analyze the internal representations, we visualize the t-SNE embeddings \cite {van2008visualizing} in Fig. \ref{fig:tsne}. The visualization reveals that the visual representations extracted by prompts corresponding to the same domain factor cluster tightly together, while those associated with distinct domain factors are well-separated. This clear clustering structure confirms the effectiveness of our hybrid contrastive learning strategy. Additionally, Fig. \ref{fig:attention} displays the domain-prompt correlation matrix, which demonstrates that the adaptive prompt orchestration mechanism correctly activates specific prompts corresponding to the dominant domain factors in real-time.

\subsection{Ablation Studies}

To better interpret the contributions of each component in CAPO, we conduct a series of ablation studies and analyses.

\textbf{(1) Hybrid Contrastive Learning Strategy.} We dismantle the training objective to verify the necessity of each contrastive loss component. We evaluate variants of CAPO trained without the text contrastive loss (w/o text), without the action temporal loss (w/o action), and without the visual contrastive loss (w/o visual). As shown in Table \ref{tab:strategy}, w/o text results in the most significant performance drop, highlighting the critical importance of semantic alignment provided by the language priors for adaptation. w/o action leads to the second-largest decline. Given that our evaluation involves diverse cross-embodiment configurations, this performance drop confirms that modeling action-conditioned temporal consistency is essential for adapting to varying agent dynamics. Notably, w/o visual shows the least performance degradation. We attribute this to the fact that the other two objectives also incorporate moderate data augmentation during training, which implicitly imparts a degree of visual invariance even without the explicit visual contrastive loss. However, the full hybrid objective still yields the best performance, demonstrating that these three components are complementary.

\begin{table*}[t!]
\centering
\newcolumntype{Y}{>{\centering\arraybackslash}X}

\setlength{\tabcolsep}{1pt} 
\renewcommand{\arraystretch}{1.2}

\begin{threeparttable}
\caption{Ablation on Hybrid Contrastive Learning Strategy for Visuomotor Policy Learning}
\label{tab:strategy}

\footnotesize 

\begin{tabularx}{\textwidth}{l *{12}{Y}}
\toprule
\multirow{2}{*}{Approach} &
\multicolumn{4}{c}{Source Domains} &
\multicolumn{4}{c}{Seen Target Domains} &
\multicolumn{4}{c}{Unseen Target Domains} \\
\cmidrule(lr){2-5} \cmidrule(lr){6-9} \cmidrule(lr){10-13}
& SR$\uparrow$ & SPL$\uparrow$ & NE$\downarrow$ & EL$\downarrow$
& SR$\uparrow$ & SPL$\uparrow$ & NE$\downarrow$ & EL$\downarrow$
& SR$\uparrow$ & SPL$\uparrow$ & NE$\downarrow$ & EL$\downarrow$ \\
\midrule

CAPO w/o visual
& 92.1$\pm$2.8 & 0.62$\pm$0.04 & 0.04$\pm$0.02 & 20$\pm$4
& 84.5$\pm$4.6 & 0.55$\pm$0.05 & 0.06$\pm$0.03 & 31$\pm$5
& 81.2$\pm$6.8 & 0.50$\pm$0.07 & 0.09$\pm$0.07 & 36$\pm$6 \\

CAPO w/o action
& 89.8$\pm$3.0 & 0.59$\pm$0.04 & 0.04$\pm$0.03 & 23$\pm$5
& 83.1$\pm$3.9  & 0.53$\pm$0.06 & 0.07$\pm$0.03 & 34$\pm$5
& 78.2$\pm$5.7  & 0.47$\pm$0.06 & 0.08$\pm$0.07 & 41$\pm$10 \\

CAPO w/o text
& 87.4$\pm$3.3 & 0.56$\pm$0.06 & 0.07$\pm$0.04 & 26$\pm$5
& 80.5$\pm$3.8 & 0.50$\pm$0.06 & 0.08$\pm$0.06 & 37$\pm$5
& 72.3$\pm$6.8 & 0.45$\pm$0.07 & 0.08$\pm$0.07 & 45$\pm$12 \\

\rowcolor{gray!15}
\textbf{CAPO}
& \textbf{97.9$\pm$1.2} & \textbf{0.66$\pm$0.04} & \textbf{0.02$\pm$0.01} & \textbf{18$\pm$3}
& \textbf{90.9$\pm$3.4} & \textbf{0.61$\pm$0.05} & \textbf{0.04$\pm$0.06} & \textbf{29$\pm$5}
& \textbf{86.4$\pm$5.7} & \textbf{0.54$\pm$0.06} & \textbf{0.06$\pm$0.07} & \textbf{32$\pm$3} \\

\bottomrule
\end{tabularx}
\end{threeparttable}
\end{table*}

\begin{figure}[t!]
    \centering
    \includegraphics[width=1.0\linewidth]{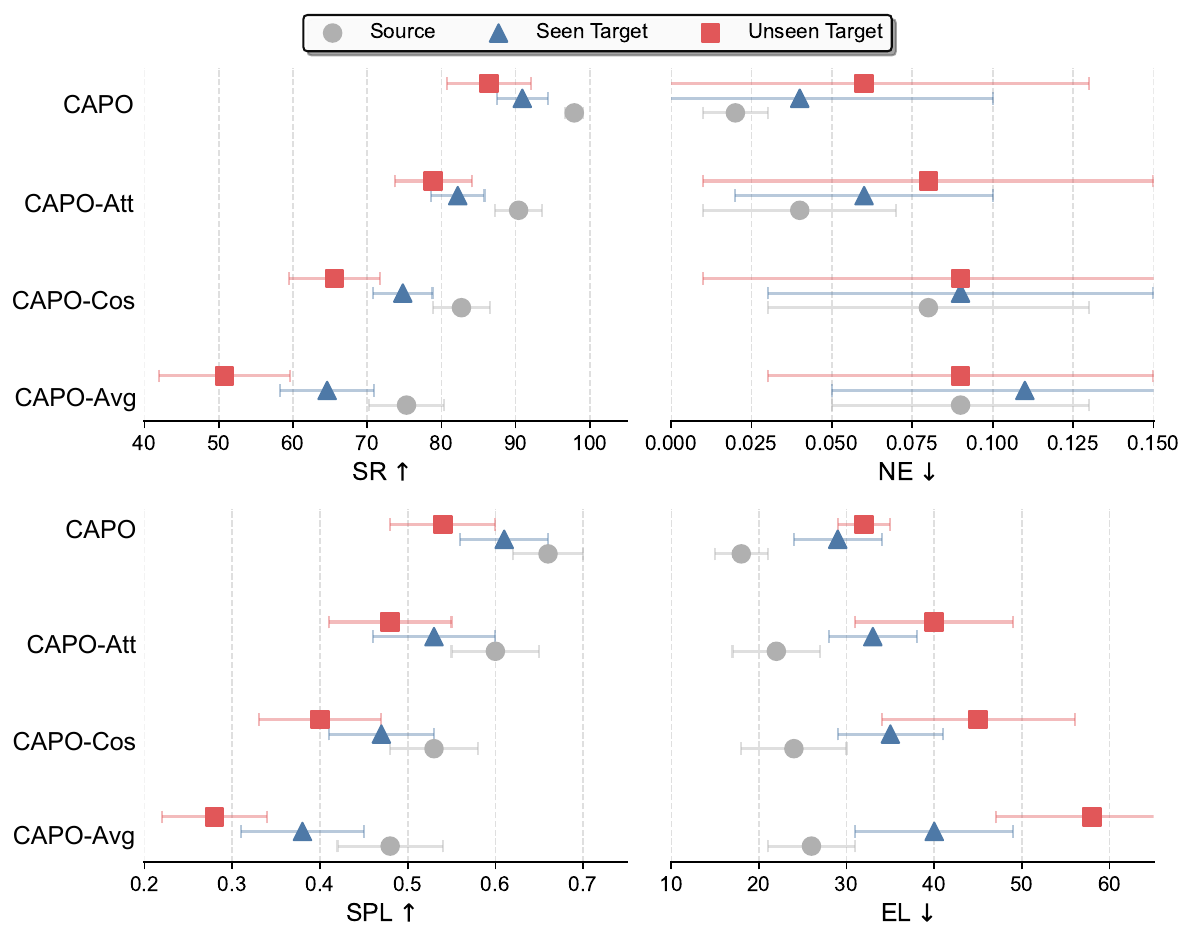}
    \caption{Ablation study on the attention mechanism. We compare the full model against uniform averaging (CAPO-Avg) and single-branch variants (CAPO-Att, CAPO-Cos). The results demonstrate that the full CAPO consistently outperforms all alternatives, confirming the necessity of combining learnable projections with cosine similarity for feature fusion.}
    \label{fig:attention_abla}
\end{figure}

\textbf{(2) Effectiveness of Attention Mechanism.} We analyze the contribution of the adaptive prompt orchestration mechanism by comparing it with three alternative fusion strategies: (i) CAPO-Avg, where the final state representation is simply the uniform mean of all prompted embeddings; (ii) CAPO-Att, which utilizes only the learnable projection network to compute attention weights, disregarding the cosine similarity calculation; and (iii) CAPO-Cos, which relies solely on the cosine similarity between prompted and vanilla embeddings to determine weights without learnable projections. The comparative results are presented in Fig. \ref{fig:attention_abla}. We observe that CAPO-AVG lags significantly behind the full approach. Both CAPO-Att and CAPO-Cos exhibit performance drops compared to the full CAPO. This decline indicates that neither the learnable projection nor the semantic similarity calculation is sufficient on its own. The superior performance of the full CAPO confirms the necessity of our hybrid design, where learnable attention captures task-relevant features while cosine similarity acts as a semantic anchor.

\begin{figure}[t!]
    \centering
    \includegraphics[width=1.0\linewidth]{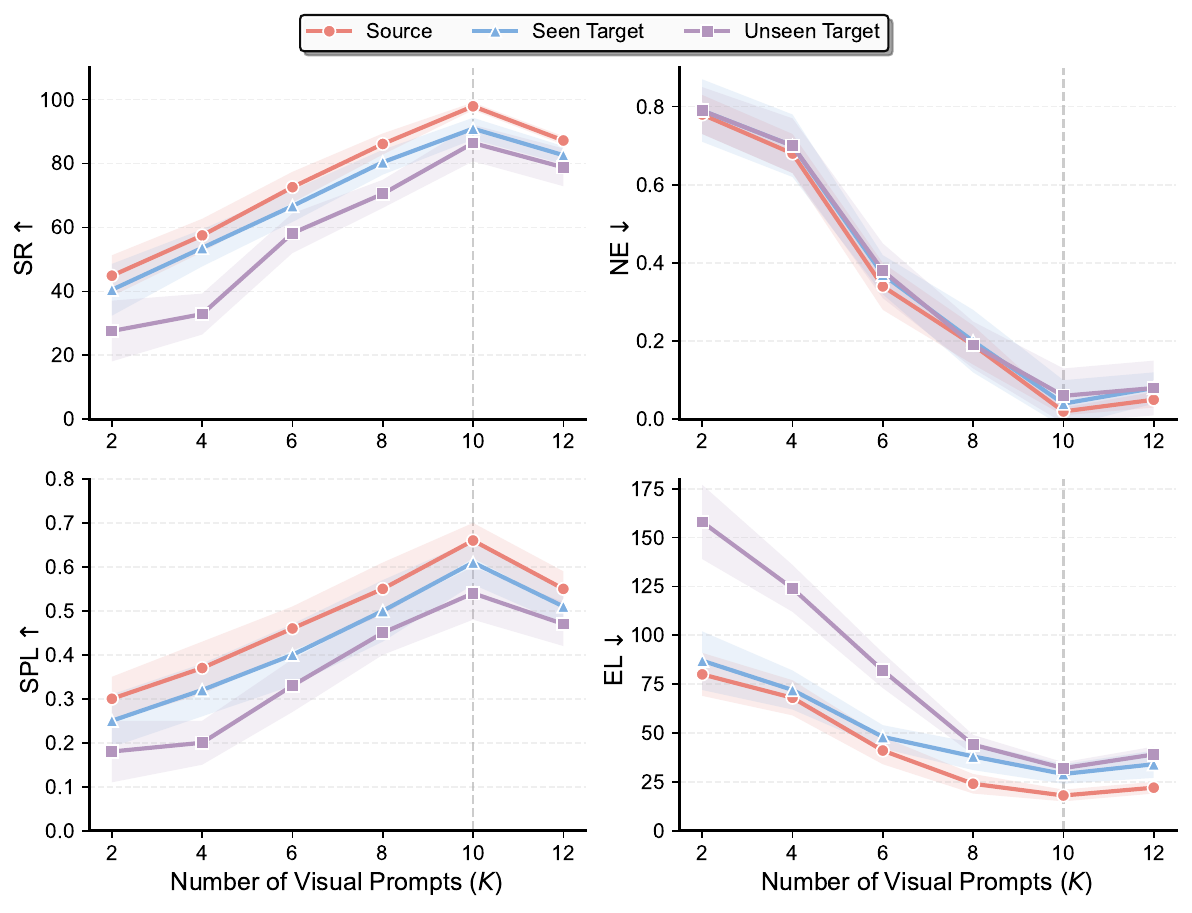}
    \caption{Sensitivity Analysis on the Prompt Pool Size $K$. Performance consistently improves as $K$ increases, peaking at $K=10$ which aligns with the distinct domain factors. Increasing $K$ further to 12 leads to degradation, indicating that excessive prompts introduce redundancy and overfitting.}
    \label{fig:prompt}
\end{figure}

\subsection{Parameter Sensitivity Analysis}
To assess the performance of CAPO under different hyperparameter settings, we further conduct a sensitivity analysis.

\textbf{(1) Sensitivity to Prompt Number.} We investigate the impact of the number of visual prompts $K$ on the policy's success rate. Specifically, we train the CAPO approach with varying prompt pool sizes $K \in \{2, 4, 6, 8, 10, 12\}$, where the setting of $K=12$ is obtained by subdividing the FOV parameter into finer intervals. The evaluation results are summarized in Fig. \ref{fig:prompt}. As shown, the performance improves consistently as $K$ increases. A small number of prompts yields limited improvement, as the capacity is insufficient to capture the diverse domain factors present in the environment. However, the performance saturates around $K=10$, which aligns with the total number of distinct domain factors defined in our data collection phase. This suggests that matching the prompt pool size to the underlying domain variations is optimal, while an excessive number of prompts may lead to redundancy without providing distinct performance gains. Based on this, we adopt $K=10$ as the default setting.

\begin{table*}[t!]
\centering
\newcolumntype{Y}{>{\centering\arraybackslash}X}

\setlength{\tabcolsep}{1pt} 
\renewcommand{\arraystretch}{1.2}

\begin{threeparttable}
\caption{Sensitivity Analysis on Semantic Regularization for Visuomotor Policy Learning}
\label{tab:noise}

\footnotesize 

\begin{tabularx}{\textwidth}{l *{12}{Y}}
\toprule
\multirow{2}{*}{Approach} &
\multicolumn{4}{c}{Source Domains} &
\multicolumn{4}{c}{Seen Target Domains} &
\multicolumn{4}{c}{Unseen Target Domains} \\
\cmidrule(lr){2-5} \cmidrule(lr){6-9} \cmidrule(lr){10-13}
& SR$\uparrow$ & SPL$\uparrow$ & NE$\downarrow$ & EL$\downarrow$
& SR$\uparrow$ & SPL$\uparrow$ & NE$\downarrow$ & EL$\downarrow$
& SR$\uparrow$ & SPL$\uparrow$ & NE$\downarrow$ & EL$\downarrow$ \\
\midrule

$\sigma = 0$
& 92.6$\pm$1.9 & 0.63$\pm$0.04 & 0.03$\pm$0.02 & 20$\pm$3
& 85.4$\pm$3.8 & 0.56$\pm$0.05 & 0.06$\pm$0.03 & 32$\pm$5
& 81.8$\pm$5.9 & 0.51$\pm$0.07 & 0.08$\pm$0.08 & 35$\pm$5 \\

\rowcolor{gray!15}
\textbf{$\sigma = 0.1$}
& \textbf{97.9$\pm$1.2} & \textbf{0.66$\pm$0.04} & \textbf{0.02$\pm$0.01} & \textbf{18$\pm$3}
& \textbf{90.9$\pm$3.4} & \textbf{0.61$\pm$0.05} & \textbf{0.04$\pm$0.06} & \textbf{29$\pm$5}
& \textbf{86.4$\pm$5.7} & \textbf{0.54$\pm$0.06} & \textbf{0.06$\pm$0.07} & \textbf{32$\pm$3} \\

$\sigma = 0.3$
& 87.1$\pm$2.1 & 0.56$\pm$0.07 & 0.04$\pm$0.04 & 27$\pm$5
& 78.7$\pm$4.2 & 0.47$\pm$0.06 & 0.08$\pm$0.08 & 38$\pm$6
& 70.5$\pm$6.6 & 0.42$\pm$0.08 & 0.10$\pm$0.08 & 45$\pm$8 \\

$\sigma = 0.5$
& 83.5$\pm$3.0 & 0.53$\pm$0.06 & 0.08$\pm$0.06 & 25$\pm$6
& 73.2$\pm$5.2 & 0.46$\pm$0.08 & 0.09$\pm$0.08 & 39$\pm$6
& 65.6$\pm$6.9 & 0.38$\pm$0.08 & 0.11$\pm$0.08 & 47$\pm$11 \\

CAPO-Reg
& 80.2$\pm$3.9 & 0.49$\pm$0.07 & 0.13$\pm$0.06 & 37$\pm$6
& 69.8$\pm$6.3 & 0.41$\pm$0.08 & 0.11$\pm$0.09 & 47$\pm$8
& 57.4$\pm$7.7 & 0.32$\pm$0.09 & 0.15$\pm$0.09 & 72$\pm$13 \\

\bottomrule
\end{tabularx}
\end{threeparttable}
\end{table*}

\begin{figure}[t!]
    \centering
    \includegraphics[width=1.0\linewidth]{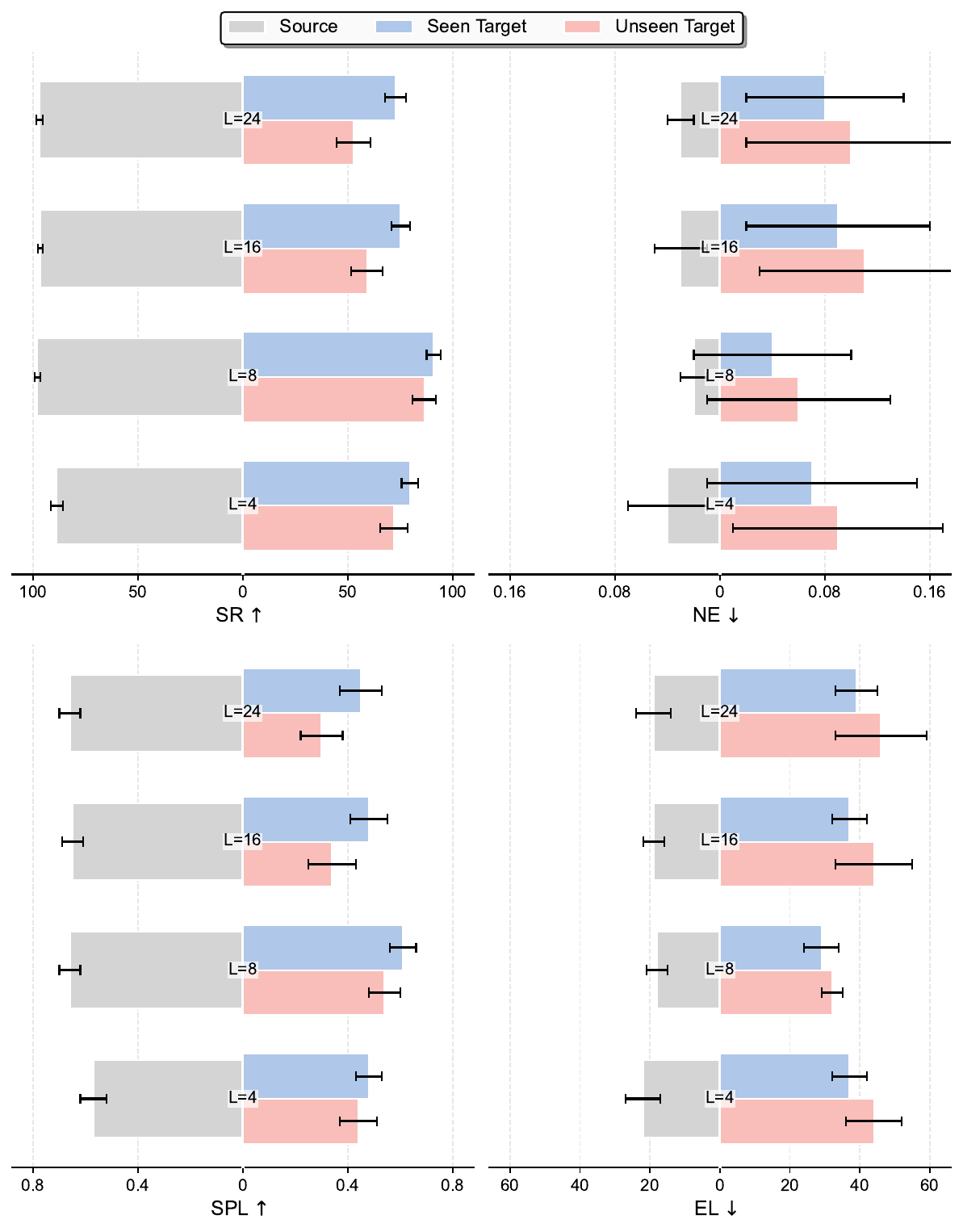}
    \caption{Sensitivity analysis on the visual prompt length $L$. We evaluate performance across varying lengths $L \in \{4, 8, 16, 24\}$. The results indicate that $L=8$ achieves the optimal trade-off between feature expressivity and adaptability. Shorter prompts ($L=4$) lead to underfitting, while longer prompts ($L=24$) result in overfitting to source domains, thereby degrading adaptability in unseen target domains.}
    \label{fig:attention_length}
\end{figure}

\textbf{(2) Sensitivity to Prompt Length.} We examine the sensitivity of CAPO to the length of the learnable visual prompts $L$, testing $L \in \{4, 8, 16, 24\}$. The results in Fig. \ref{fig:attention_length} indicate that $L=8$ provides the optimal balance. Shorter prompts ($L=4$) lack the expressivity to encode complex domain-specific features, resulting in underfitting. Conversely, longer prompts ($L=24$) tend to overfit the specific visual textures of the source domains, hampering zero-shot adaptation to unseen target domains. This suggests that a moderate prompt length is sufficient and effective for steering the frozen CLIP encoder.

\textbf{(3) Sensitivity to Semantic Regularization.} We examine the sensitivity of CAPO to the noise scale $\sigma$ used in the semantic regularization mechanism $\zeta \sim \mathcal{N}(0, \sigma^2 I)$. As shown in Table \ref{tab:noise}, introducing a moderate amount of noise ($\sigma=0.1$) yields the highest performance, outperforming other settings. This suggests that appropriate regularization helps prevent overfitting to source domains. However, performance degrades when $\sigma$ is increased from 0.3 to 0.5, indicating that excessive perturbation disrupts feature integrity. Additionally, removing text prompts while relying solely on regularization (CAPO-Reg) results in significant performance drops, confirming that noise injection is complementary to, but cannot replace, the explicit semantic guidance from text prompts.

\section{Conclusion}
In this work, we propose CAPO, a novel approach that integrates contrastive prompt learning with adaptive prompt orchestration to enhance the sample efficiency and cross-embodiment adaptation of navigation policies. For contrastive prompt learning, a hybrid contrastive strategy incorporating visual, action, and text objectives is employed to learn a diverse pool of domain-specific visual prompts. These prompts effectively capture fine-grained environmental and embodiment variations. After this, the visual encoder is frozen, and a dynamic prompt orchestration mechanism is introduced to adaptively aggregate these prompts based on the current observation context. CAPO allows the agent to dynamically re-weight visual features under domain variations, enabling rapid adaptation to unseen configurations. Extensive simulation experiments demonstrate that CAPO significantly improves sample efficiency and asymptotic performance, while enabling robust zero-shot adaptation across diverse unseen domains. However, the approach's adaptability is bounded by the diversity of source domains, potentially limiting performance under extreme distribution shifts orthogonal to the learned prompt pool. Future work will address this via online prompt adaptation. Moreover, while the current approach utilizes semantic priors primarily for the ObjectNav task, we plan to extend this semantic alignment capability to more complex vision-language navigation tasks.

\bibliographystyle{IEEEtran}
\bibliography{IEEEabrv,yuhang}

\vspace{-1cm}
\begin{IEEEbiography}[{\includegraphics[width=1in,height=1.25in,clip,keepaspectratio]{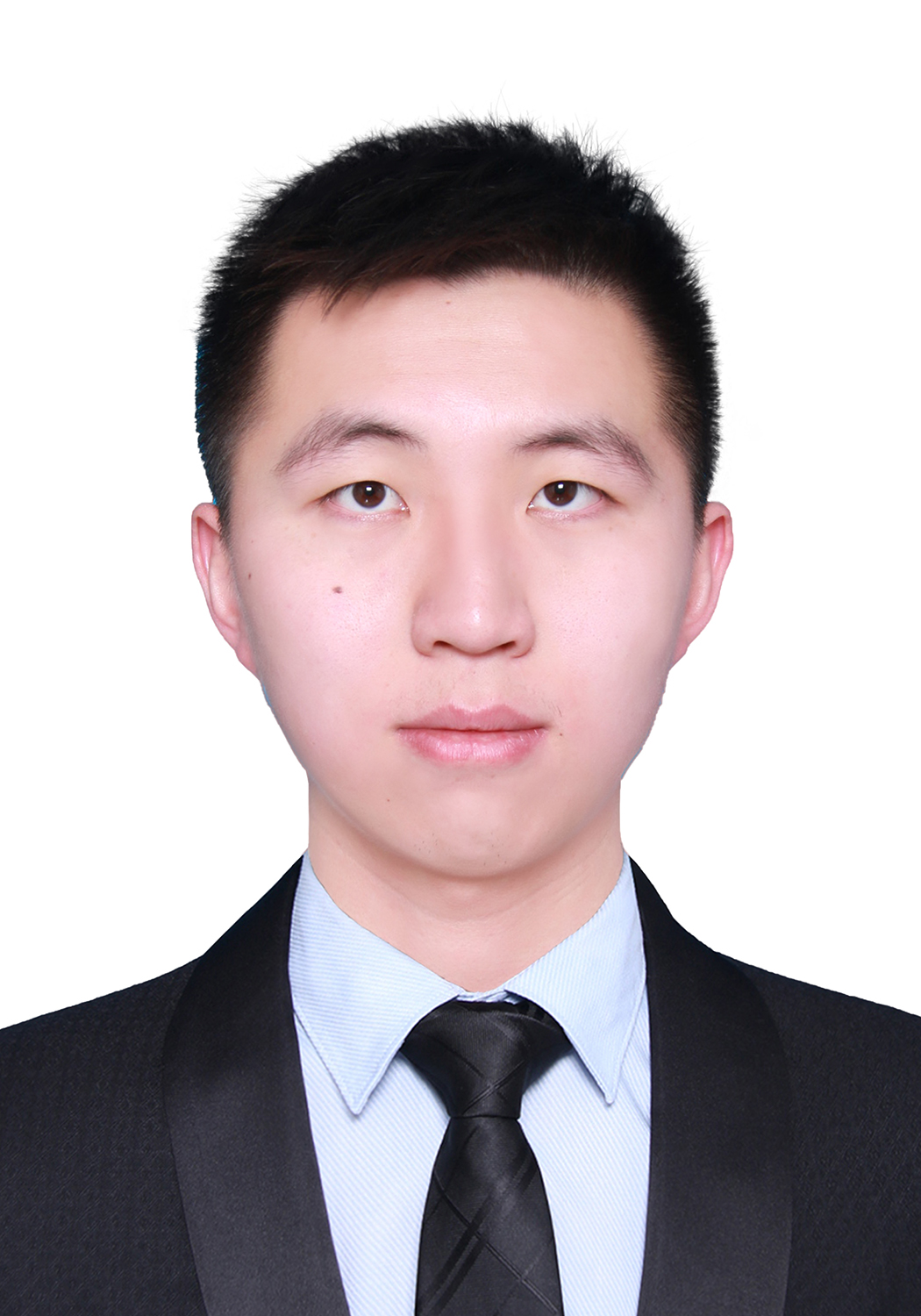}}]{Yuhang Zhang} (graduate student member, IEEE) received the B.E. degree in flight vehicle propulsion engineering from Harbin Engineering University, Harbin, China, in 2021 and the M.Eng. degree from the School of Mechanical \& Aerospace Engineering at Nanyang Technological University (NTU), Singapore, in 2023. Currently,  he is pursuing his Ph.D. degree at the School of Mechanical \& Aerospace Engineering at NTU, Singapore.

His research primarily focuses on unmanned aerial vehicles, deep reinforcement learning, and vision-and-language navigation.
\end{IEEEbiography}

\begin{IEEEbiography}[{\includegraphics[width=1in,height=1.25in,clip,keepaspectratio]{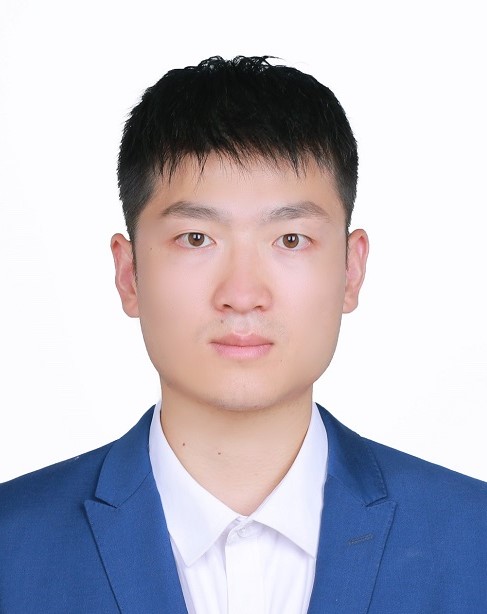}}]{Chao Yan} (member, IEEE) received the B.E. degree in electrical engineering and automation from China University of Mining and Technology, Xuzhou, China, in 2017, and the M.S. and Ph.D. degrees in control science and engineering from the National University of Defense Technology, Changsha, China, in 2019, and 2023, respectively. He was a visiting Ph.D. student with the School of Mechanical and Aerospace Engineering, Nanyang Technological University, Singapore, from 2021 to 2022.

He is currently an Associate Professor with the College of Automation Engineering, Nanjing University of Aeronautics and Astronautics, Nanjing, China. His research interests include deep reinforcement learning and coordination control of UAV swarms.
\end{IEEEbiography}

\begin{IEEEbiography}[{\includegraphics[width=1in,height=1.25in,clip,keepaspectratio]{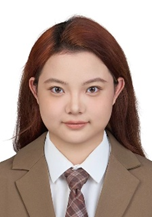}}]{Jiaxi Yu} received the B.E. degree in Mechanical Design, Manufacturing and Automation from Shandong University, Jinan, China, in 2025. She is presently enrolled in the M.SC. program at the School of Mechanical and Aerospace Engineering, Nanyang Technological University (NTU), Singapore.

Her research interests are mainly oriented toward reinforcement learning.
\end{IEEEbiography}

\begin{IEEEbiography}[{\includegraphics[width=1in,height=1.25in,clip,keepaspectratio]{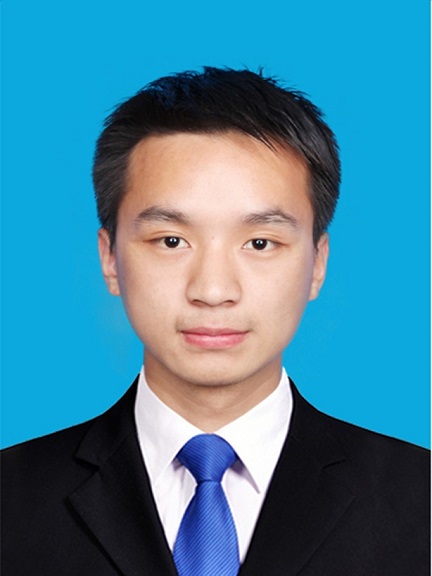}}]{Jiaping Xiao} (member, IEEE)
received the B.E. degree in aircraft design and engineering and the M.S. degree in flight dynamics and control from Beihang University, Beijing, China, in 2014 and 2017, and the Ph.D. degree in intelligent systems from Nanyang Technological University (NTU), Singapore, in 2024. He is currently a research fellow with the School of Mechanical and Aerospace Engineering, NTU.

His research interests include cyber-physical systems, reinforcement learning, machine vision, and aerial robotics.
\end{IEEEbiography}

\begin{IEEEbiography}[{\includegraphics[width=1in,height=1.25in,clip,keepaspectratio]{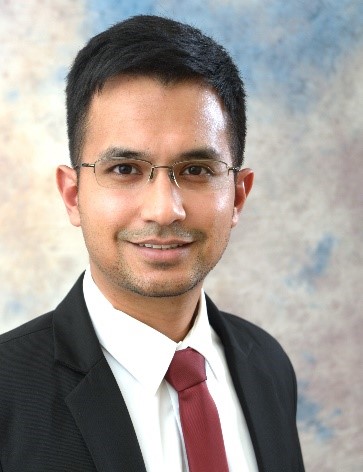}}]{Mir Feroskhan}
(member, IEEE) received B.E. degree (Hons.) in aerospace engineering from Nanyang Technological University, Singapore, in 2011, and the Ph.D. degree in aerospace engineering from the Florida Institute of Technology, Melbourne, FL, in 2016. He is currently an assistant professor with the School of Mechanical \& Aerospace Engineering at NTU.

His research interests include nonlinear control systems, multi-agent systems, flight dynamics and control, and aerial robotics.
\end{IEEEbiography}

\end{document}